%% file: anonymous-submission-latex-2024.tex
\pdfoutput=1
\documentclass[letterpaper]{article} 
\usepackage{aaai24}  
\usepackage{times}  
\usepackage{helvet}  
\usepackage{courier}  
\usepackage[hyphens]{url}  
\usepackage{graphicx} 
\usepackage{natbib}  
\usepackage{caption} 
\usepackage{xcolor}
\usepackage{algorithm}
\usepackage{algorithmic}
\usepackage[utf8]{inputenc} 
\usepackage[T1]{fontenc}    
\usepackage{url}            
\usepackage{booktabs}       
\usepackage{amsfonts}       
\usepackage{nicefrac}       
\usepackage{bm}

\usepackage{microtype}      
\usepackage{subcaption}
\usepackage{todonotes}
\usepackage{lipsum}
\usepackage{mwe}

\newcommand{\Robust}{Adversarially Robust\xspace}
\newcommand{\algor}{algorithmic recourse\xspace}
\newcommand{\AlgoR}{Algorithmic Recourse\xspace}
\newcommand{\vanilla}{non-robust\xspace}
\newcommand{\robust}{adversarially robust\xspace}
\newcommand{\advexs}{adversarial examples\xspace}
\newcommand{\chvae}{{C-CHVAE}\xspace}
\usepackage{amssymb,amsmath,amsthm}
\newtheorem{definition}{Definition}
\newtheorem{theorem}{Theorem}
\newtheorem{lemma}{Lemma}

\newenvironment{hproof}{%
  \proof}{\endproof}

\newcommand{\ie}{i.e.\xspace}
\newcommand{\eg}{e.g.\xspace}

\newcommand{\hideh}[1]{}

\input{math_commands.tex}

\input{notation.tex}

\usepackage{colortbl}

\definecolor{Gray}{gray}{0.9}
\definecolor{LightCyan}{rgb}{0.88,1,1}

\usepackage{listings}
\DeclareCaptionStyle{ruled}{labelfont=normalfont,labelsep=colon,strut=off} 
\floatstyle{ruled}
\newfloat{listing}{tb}{lst}{}
\floatname{listing}{Listing}
\title{On the Trade-offs between Adversarial Robustness and Actionable Explanations}
\author {
    Satyapriya Krishna\textsuperscript{\rm 1},
    Chirag Agarwal\textsuperscript{\rm 1},
    Himabindu Lakkaraju\textsuperscript{\rm 1}
}
\affiliations {
    \textsuperscript{\rm 1} Harvard University\\
    skrishna@g.harvard.edu
}

\usepackage{bibentry}

\begin{document}

\maketitle

\begin{abstract}
As machine learning models are increasingly being employed in various high-stakes settings, it becomes important to ensure that predictions of these models are not only adversarially robust, but also readily explainable to relevant stakeholders. However, it is unclear if these two notions can be simultaneously achieved or if there exist trade-offs between them. In this work, we make one of the first attempts at studying the impact of \robust models on actionable explanations which provide end users with a means for recourse. We theoretically and empirically analyze the cost (ease of implementation) and validity (probability of obtaining a positive model prediction) of recourses output by state-of-the-art algorithms when the underlying models are \robust vs. \vanilla. More specifically, we derive theoretical bounds on the differences between the cost and the validity of the recourses generated by state-of-the-art algorithms for \robust vs. \vanilla linear and non-linear models. Our empirical results with multiple real-world datasets validate our theoretical results and show the impact of varying degrees of model robustness on the cost and validity of the resulting recourses. Our analyses demonstrate that \robust models significantly increase the cost and reduce the validity of the resulting recourses, thus shedding light on the inherent trade-offs between adversarial robustness and actionable explanations.
\end{abstract}

\input{010intro}
\input{020relwork}

\input{030prelims}
\input{040analysis}
\input{050experiments}
\input{060conclusions}

\bibliography{aaai24}
\appendix
\input{070appendix}
\end{document}

%% file: math_commands.tex
\usepackage{amsmath,amsfonts,bm}

\def\eqref#1{equation~\ref{#1}}

\def\1{\bm{1}}

\DeclareMathAlphabet{\mathsfit}{\encodingdefault}{\sfdefault}{m}{sl}
\SetMathAlphabet{\mathsfit}{bold}{\encodingdefault}{\sfdefault}{bx}{n}

\DeclareMathOperator*{\argmin}{arg\,min}

%% file: notation.tex
\def\to{{\,\rightarrow\,}}

\mathchardef\mhyphen="2D

\newcommand{\norm}[1]{{ \left\lVert#1\right\rVert }}

\newcommand{\vertiii}[1]{{\left\vert\kern-0.25ex\left\vert\kern-0.25ex\left\vert #1
    \right\vert\kern-0.25ex\right\vert\kern-0.25ex\right\vert}}

\def\bm{{\mathbf{m}}}

\def\bw{{\mathbf{w}}}
\def\bx{{\mathbf{x}}}

\def\bz{{\mathbf{z}}}

\def\bI{{\mathbf{I}}}

\def\bK{{\mathbf{K}}}

\def\bX{{\mathbf{X}}}
\def\bY{{\mathbf{Y}}}

\def\cD{\mathcal{D}}

\def\cG{\mathcal{G}}

\def\cI{\mathcal{I}}

\def\cL{\mathcal{L}}

\def\cX{\mathcal{X}}

\def\cZ{\mathcal{Z}}



%% file: 010intro.tex
\section{Introduction}
\label{sec:intro}

In recent years, machine learning (ML) models have made significant strides, becoming indispensable tools in high-stakes domains such as banking, healthcare, and criminal justice. As these models continue to gain prominence, it is more crucial than ever to address the dual challenge of providing actionable explanations to individuals negatively impacted (e.g., denied loan applications) by model predictions, and ensuring adversarial robustness to maintain model integrity. 
Both prior research and recent regulations have emphasized the importance of adversarial robustness and actionable explanations, deeming them as key pillars of trustworthy machine learning~\citep{GDPR, hamon2020robustness,voigt2017eu} that are critical to real-world applications. 

Existing machine learning research has explored adversarial robustness and actionable explanations in different silos, where several techniques have been proposed for generating actionable explanations in practice (using counterfactual explanations)~\citep{wachter2017counterfactual,Ustun2019ActionableRI,pawelczyk2020learning,karimi2019model} and adversarial examples~\citep{szegedy2013intriguing,goodfellow2014explaining}. The goal of a counterfactual explanation is to generate recourses for individuals impacted by model outcomes, \eg when an individual is denied a loan by a predictive model, a counterfactual explanation informs them about which feature should be changed and by how much in order to obtain a positive outcome from the model. However, an adversarial attack aims to demonstrate the vulnerabilities of modern deep neural network (DNN) models and generates infinitesimal input perturbations to achieve adversary-selected model outcomes. Recently, adversarial training has been proposed as a defense against adversarial examples, with the goal of training robust models in adversarial scenarios~\citep{madry2017towards}. However, its potential impact on generating algorithmic recourse has not been studied yet, and this is the focus of our work.

Despite recent regulations~\citep{GDPR,hamon2020robustness} emphasizing the importance of adversarial robust models and generating actionable explanations for end-users, there has been little to no work in understanding the connection between these contrasting trustworthy properties. While previous works~\citep{xie2020adversarial} study the trade-offs between adversarial robustness and accuracy on the original data, the trade-offs between adversarial robustness and recourse have been unexplored, and ours is one of the first works to shed light on these trade-offs across diverse models and datasets. To this end, only a few works~\citep{bansal2020sam,shafahi2019adversarial} show the impact of robust models on model gradients, where they find that input gradients are less noisier for adversarially robust models as compared to non-robust models. While these findings highlight the impact of robust models on model gradients, we aim to show that adversarial robustness does not come for free, and there are trade-offs between robustness and recourse costs.

\paragraph{Present work.} In this study, we address the aforementioned gaps by presenting the first-ever investigation of the impact of adversarially robust models on algorithmic recourse. We provide a theoretical and empirical analysis of the \textit{cost} (ease of implementation) and \textit{validity} (likelihood of achieving a desired model prediction) of the recourses generated by state-of-the-art algorithms for \robust and \vanilla models. In particular, we establish theoretical bounds on the differences in cost and validity for recourses produced by various gradient-based~\citep{laugel2017inverse,wachter2017counterfactual} and manifold-based~\citep{pawelczyk2020learning} recourse methods for \robust and \vanilla linear and non-linear models (see Section~\ref{sec:theory}). To achieve this, we first derive theoretical bounds on the differences between the weights (parameters) of \robust and \vanilla linear and non-linear models, and then use these bounds to establish the differences in cost and validity of the corresponding recourses. It is important to note that in this study, we analyze the impact on algorithmic recourses when the model is made adversarially robust. This should not be confused with the analysis comparing adversarial examples and algorithmic recourses\cite{freiesleben2022intriguing,pawelczyk2020learning}.

We conducted extensive experiments with multiple real-world datasets from diverse domains (Section~\ref{sec:expt}). Our theoretical and empirical analyses provide several interesting insights into the relationship between adversarial robustness and algorithmic recourse: i) the cost of recourse increases with the degree of robustness of the underlying model, and ii) the validity of recourse deteriorates as the degree of robustness of the underlying model increases. Additionally, we conducted a qualitative analysis of the recourses generated by state-of-the-art methods, and observed that the number of valid recourses for any given instance decreases as the underlying model's robustness increases. More broadly, our analyses and findings shed light on the the inherent trade-offs between adversarial robustness and actionable explanations.

%% file: 020relwork.tex
\section{Related Work}
\label{sec:related}

\paragraph{Algorithmic Recourse. } Several approaches have been proposed in recent literature to provide recourses to affected individuals~\citep{wachter2017counterfactual,Ustun2019ActionableRI,van2019interpretable,pawelczyk2020learning,mahajan2019preserving,karimi2019model,karimi2020probabilistic}. These approaches can be broadly categorized along the following dimensions \cite{verma2020counterfactual}: 
\emph{type of the underlying predictive model} (e.g., tree-based vs. differentiable classifier), \emph{type of access} of the underlying predictive model (e.g., black box vs. gradient access), whether they encourage \emph{sparsity} in counterfactuals (i.e., allowing changes in a small number of features), whether counterfactuals should lie on the \emph{data manifold}, whether the underlying \emph{causal relationships} should be accounted for when generating counterfactuals, and whether the produced output by the method should be \emph{multiple diverse counterfactuals} or a single counterfactual. Recent works also demonstrate that recourses output by state-of-the-art techniques are not robust, i.e., small perturbations to the original instance~\citep{dominguezolmedo2021adversarial,slack2021counterfactual}, the underlying model~\citep{upadhyay2021robust,rawal2021modelshifts}, or the recourse~\citep{pawelczyk2022algorithmic} itself may render the previously prescribed recourse(s) invalid. These works also proposed minimax optimization problems to find \emph{robust} recourses to address the aforementioned challenges.

\paragraph{Adversarial Examples and Robustness. } Prior works have shown that complex machine learning model, such as deep neural networks, are vulnerable to small changes in input~\citep{szegedy2013intriguing}. This behavior of predictive models allows for generating \advexs (AEs) by adding infinitesimal changes to input targeted to achieve adversary-selected outcomes~\citep{szegedy2013intriguing, goodfellow2014explaining}. Prior works have proposed several techniques to generate AEs using varying degrees of access to the model, training data, and the training procedure~\citep{chakraborty2018adversarial}. While gradient-based methods~\citep{goodfellow2014explaining, kurakin2016adversarial} return the smallest input perturbations which flip the label as \advexs, generative methods~\citep{zhao2017generating} constrain the search for \advexs to the training data-manifold. Finally, some methods~\citep{cisse2017houdini} generate \advexs for non-differentiable and non-decomposable measures in complex domains such as speech recognition and image segmentation. Prior works have shown that Empirical Risk Minimization (ERM) does not yield models that are robust to \advexs~\citep{goodfellow2014explaining,kurakin2016adversarial}. Hence, to reliably train \robust models,~\citet{madry2017towards} proposed the adversarial training objective which minimizes the worst-case loss within some $\epsilon$-ball perturbation region around the input instances.

\paragraph{Intersections between Adversarial ML and Model Explanations. } There has been a growing interest in studying the intersection of adversarial ML and model explainability~\citep{hamon2020robustness}. 
Among all these works, two are relevant to our work~\citep{pawelczyk2021connections,shah2021input}. \citet{shah2021input} studied the interplay between adversarial robustness and post hoc explanations~\citep{shah2021input}, demonstrating that gradient-based explanations violate the primary assumption of attributions -- features with higher attribution are more important for model prediction -- in case of \vanilla models. Further, they show that such a violation does not occur when the underlying models are robust to $\ell_{2}$ and $\ell_{\infty}$ input perturbations. More recently,~\citet{pawelczyk2021connections} demonstrated that the distance between the recourses generated by state-of-the-art methods and adversarial examples is small for linear models. While existing works explore the connections between adversarial ML and model explanations, none focus on the trade-offs between adversarial robustness and actionable explanations, which is the focus of our work.

%% file: 030prelims.tex
\section{Preliminaries}
\label{sec:prelim}

\paragraph{Notation}
In this work, we denote a model $f: \mathbb{R}^{d} \to \mathbb{R}$, where $\bx \in \cX$ is a $d$-dimensional input sample, $\cX$ is the training dataset, and the model is parameterized with weights $\bw$. In addition, we represent the \vanilla and \robust models using $f_{\text{NR}}(\bx)$ and $f_{\text{R}}(\bx)$, and the linear and neural network models using $f^{\text{L}}(\bx)$ and $f^{\text{NTK}}(\bx)$. Below, we provide a brief overview of adversarially robust models, and some popular methods for generating recourses.

\paragraph{\Robust Models.} Despite the superior performance of machine learning (ML) models, they are susceptible to adversarial examples (AEs), i.e., inputs generated by adding infinitesimal perturbations to the original samples targeted to change prediction label~\citep{agarwal2019improving}. One standard approach to ameliorate this problem is via adversarial training which minimizes the worst-case loss within some perturbation region (the perturbation model)~\citep{madryzico}. In particular, for a model $f$ parameterized by weights $\bw$, loss function $\ell(\cdot)$, and training data $\{\mathbf{x}_{i}, y_{i}\}_{i=\{1,2,\dots,n\}} \in \mathcal{D}_{\textup{train}}$, the optimization problem of minimizing the worst-case loss within $\ell_{p}-$norm perturbation with radius $\epsilon$ is:
\begin{equation}
    \min_{\bw} \frac{1}{|\cD_{\text{train}}|}\sum_{(x,y)\in \cD_{\text{train}}} \max_{\delta \in \Delta_{p,\epsilon}} \ell(f(\bx + \delta)), y),
\end{equation}
where $\cD_{\text{train}}$ denotes the training dataset and $\Delta_{p,\epsilon} = \{\delta: \|\delta\|_{p} \leq \epsilon \}$ is the $\ell_{p}$ ball with radius $\epsilon$ centered around sample $\bx$. We use $p = \infty$ for our theoretical analysis resulting in a closed-form solution of the model parameters $\bw$.

\paragraph{Algorithmic Recourse.}  One way to generate recourses is by explaining to affected individuals what features in their profile need to change and by how much in order to obtain a positive outcome. Counterfactual explanations that essentially capture the aforementioned information can be therefore used to provide recourses. The terms \textit{counterfactual explanations} and \textit{algorithmic recourse} have, in fact, become synonymous in recent literature~\citep{karimi2020survey, Ustun2019ActionableRI,Venkatasubramanian2020}. In particular, methods that try to find algorithmic recourses do so by finding a counterfactual $\bx'=\bx+\zeta$ that is closest to the original instance $\bx$ and change the model's prediction $f(\bx+\zeta)$ to the target label, where $\zeta$ determines a set of changes that can be made to $\bx$ in order to reverse the negative outcome. Next, we describe three popular recourse methods we analyze to understand the implications of adversarially robust models on algorithmic recourses.

\paragraph{Score CounterFactual Explanations (SCFE).} For a given model $f: \mathbb{R}^{d} \to \mathbb{R}$, a distance function $d: \mathbb{R}^{d} \times \mathbb{R}^{d} \to \mathbb{R}_{+}$, and sample $\bx$,~\citet{wachter2017counterfactual} define the problem of generating a counterfactual $\bx'{=}\bx+\zeta$ using the following objective:
\begin{equation}
    \argmin_{\bx'}~~(f(\bx') - s)^2 + \lambda d(\bx', \bx),
    \label{eq:scfe}
\end{equation}
where $s$ is the target score for the counterfactual $\bx'$, $\lambda$ is the regularization coefficient, and $d(\cdot)$ is the distance between sample $\bx$ and its counterfactual $\bx'$.

\paragraph{C-CHVAE.} Given a Variational AutoEncoder (VAE) model with encoder $\cI_{\gamma}$ and decoder $\cG_{\theta}$ trained on the original data distribution $\cD_{\textup{train}}$, C-CHVAE~\citep{pawelczyk2020learning} aims to generate recourses in the latent space $\cZ$, where $\cI_{\gamma}: \cX \to \cZ$. The encoder transforms a given sample $\bx$ into a latent representation $\bz \in \cZ$ and the decoder takes $\bz$ as input and generates $\hat{\bx}$ as similar as possible to $\bx$. Formally, \chvae generates recourse using the following objective function:
\begin{equation}
    \zeta^{*} = \argmin_{\zeta \in \cZ} \|\zeta\| ~~~ \text{such that} ~~~ f(\cG_{\theta}(\cI_{\gamma}(\bx) + \zeta)) \not= f(\bx),
\end{equation}
where $\zeta$ is the cost for generating a recourse, $\cI_{\gamma}$ allows to search for counterfactuals in the data manifold and $\cG_{\theta}$ projects the latent counterfactuals back to the input feature space.

\paragraph{Growing Spheres Method (GSM).} While the above techniques directly optimize specific objective functions for generating counterfactuals, 
GSM~\citep{laugel2017inverse} uses a search-based algorithm to generate recourses by randomly sampling points around the original instance $\bx$ until a sample with the target label is found. In particular, GSM method involves first drawing an $\ell_{2}$-sphere around a given instance $\bx$, randomly samples point within that sphere and checks whether any sampled points result in target prediction. This method then contracts or expands the sphere until a (sparse) counterfactual is found and returns it. GSM defines a minimization problem using a function $c: \cX \times \cX \to \mathbb{R}_{+}$, where $c(\bx, \bx')$ is the cost of going from instance $\bx$ to counterfactual $\bx'$.
\begin{equation}
    {\bx'}^{*} = \argmin_{\bx' \in \cX} \{ c(\bx, \bx')~~|~~f(\bx') \neq f(\bx) \},
\end{equation}
where $\bx'$ is sampled from the $\ell_{2}$-ball around $\bx$ such that $f(\bx'){\neq}f(\bx)$, $c(\bx, \bx') {=} \| \bx' - \bx\|_{2} + \gamma\|\bx' - \bx\|_{0}$, and $\gamma \in \mathbb{R}_{+}$ is the weight associated to the sparsity objective.

%% file: 040analysis.tex
\section{Our Theoretical Analysis}
\label{sec:theory}

Here, we perform a detailed theoretical analysis to bound the cost and validity differences of recourses generated by state-of-the-art methods when the underlying models are \vanilla vs. \robust, for linear and non-linear predictors. In particular, we compare the cost differences (Sec.~\ref{sec:cost}) of the recourses obtained using 1) gradient-based methods like SCFE ~\citep{wachter2017counterfactual} and 2) manifold-based methods like \chvae~\citep{pawelczyk2020learning}. Finally, we show that the validity of the recourses generated using existing methods for robust models is lower compared to that of \vanilla models (Sec.~\ref{sec:validity}).

\subsection{Cost Analysis}
\label{sec:cost}

The cost of a generated algorithmic recourse is defined as the distance between the input instance $\bx$ and the counterfactual $\bx'$ obtained using a recourse finding method~\citep{verma2020counterfactual}. Algorithmic recourses with lower costs are considered better as they achieve the desired outcome with minimal changes to input. Next, we theoretically analyze the cost difference of recourses generated for \vanilla and \robust linear and non-linear models.
Below, we first find the weight difference between \vanilla and \robust models and then use these lemmas to derive the recourse cost differences.

\paragraph{Cost Analysis of recourses generated using SCFE method} Here, we derive the lower and upper bound for the cost difference of recourses generated using SCFE~\citep{wachter2017counterfactual} method when the underlying models are \vanilla vs. \robust linear and non-linear models. We first derive a bound for the difference between \vanilla and \robust linear model weights.
\begin{lemma}(Difference between \vanilla and \robust linear model weights) For an instance $\bx$, let $\bw_{\textup{NR}}$ and $\bw_{\textup{R}}$ be weights of the \vanilla and \robust linear model. Then, for a normalized Lipschitz activation function $\sigma(\cdot)$, the difference in the weights can be bounded as:
\begin{equation}
    \|\mathbf{w}_{\textup{NR}} - \mathbf{w}_{\textup{R}}\|_2 \leq \Delta
\end{equation}

where $\Delta = N\eta ( y\|\mathbf{x}^{T}\|_{2} + \epsilon \sqrt{d})$, $\eta$ is the learning rate, $\epsilon$ is the $\ell_{2}$-norm perturbation ball around the sample $\bx$, $y$ is the label for $\bx$, $N$ is the total number of training epochs, and $d$ is the dimension of the input features. Subsequently, we show that $\|\mathbf{w}_{\textup{NR}}\|_2 - \Delta \leq \|\mathbf{w}_{\textup{R}}\|_2 \leq \|\mathbf{w}_{\textup{NR}}\|_2 + \Delta$.
\label{thm:weight-linear-sketch}
\end{lemma}
\begin{hproof}
    We separately derive the gradients for updating the weight for the \vanilla and \robust linear models. The proof uses sigmoidal and triangle inequality properties to derive the bound for the difference between the \vanilla and \robust linear model. In addition, we use reverse triangle inequality properties to show that the weights of the \robust linear model are bounded by $\|\bw_{\text{NR}}\|_2\pm\Delta$. See Appendix~\ref{sec:proof-lemma1-full}
    for detailed proof.
\end{hproof}

\textit{Implications:} We note that the weight difference in Eqn.~\ref{thm:weight-linear-sketch} is proportional to the $\ell_{2}$-norm of the input and the square root of the number of dimensions of $\bx$. In particular, the bound is tighter for samples with lower feature dimensions $d$ and models with a smaller degree of robustness $\epsilon$.

Next, we define the closed-form solution for the cost $\zeta^{*}$ to generate a recourse for the linear model.

\begin{definition} (Optimal cost for linear models~\citep{pawelczyk2021connections}) For a given scoring function $f(\bx){=}\bw^{T}\bx$, the SCFE method generates
a recourse for an input $\bx$ using cost $\zeta$ such that:
\begin{equation}
    \zeta^{*} = m \frac{\lambda}{\lambda + \|\bw\|_{2}^2} \cdot \bw,
\end{equation}
where $m=s-f(\bx')$ is the target residual, $s$ is the target score for $\bx$, $\bw$ is the weight of the linear model, and $\lambda$ is a given hyperparameter.
\label{thm:definition-optimal}
\end{definition}
We now derive the cost difference bounds of recourses generated using SCFE when the underlying model is \vanilla and \robust linear models.
\begin{theorem}  (Cost difference of SCFE for linear models) For a given instance $\bx$, let $\bx'_{\textup{NR}}=\bx+\zeta_{\textup{NR}}$ and $\bx'_{\textup{R}}=\bx+\zeta_{\textup{R}}$ be the recourse generated using Wachter's algorithm for the \vanilla and \robust linear models. Then, for a normalized Lipschitz activation function $\sigma(\cdot)$, the difference in the recourse for both models can be bounded as:
\begin{equation}
    \|\zeta_{\textup{NR}}\|_2 - \|\zeta_{\textup{R}}\|_2
    \leq \Big|~\lambda\frac{2\norm{\bw_{\textup{NR}}}_{2} + \Delta}{\norm{\bw_{\textup{NR}}}_{2} (\|\mathbf{w}_{\textup{NR}}\|_2 - \Delta)}~\Big|,
    \label{eq:scfe_bound}
\end{equation}
where $\bw_{\textup{NR}}$ is the weight of the \vanilla model, $\lambda$ is the scalar coefficient on the distance between original sample $\bx$ and generated counterfactual $\bx'$, and $\Delta$ is defined in Lemma~\ref{thm:weight-linear-sketch}.
\label{thm:cost-bound-linear-sketch}
\end{theorem}
\begin{hproof}

We use the optimal cost for recourses in linear models (see Def.~\ref{thm:definition-optimal}) for deriving the cost difference bounds. The proof for the weight difference uses linear algebra and triangle inequality properties. See Appendix~\ref{sec:proof-thm1-full} for the complete proof.
\end{hproof}
\textit{Implications:} The derived bounds imply that the differences between costs are a function of the quantity $\Delta$ (RHS term from Lemma~\ref{thm:weight-linear-sketch}), the weights of the non-robust model $||\mathbf{w}_{\text{NR}}||_{2}$, and $\lambda$, where the bound of the difference between costs is tighter (lower) for smaller $\Delta$ values and when the $\ell_{2}$-norm of the non-robust model weight is large (due to the quadratic term in the denominator). We note that the $\Delta$ term is a function of the $\ell_{2}$-norm of the input $\bx$ and the square root of the number of dimensions $d$ of the input sample, where the bound is tighter for smaller feature dimensions $d$, models with a smaller degree of robustness $\epsilon$, and $\bx$ with larger $\ell_{2}$-norms.

Next, we define the closed-form solution for the cost $\zeta^{*}$ required to generate a recourse when the underlying model is a wide neural network.
\begin{definition} (Kernel Matrix for ReLU networks~\citep{du2018gradient, zhang2022rethinking}) The closed-form solution of the Neural Tangent Kernel for a two-layer neural network model with ReLU non-linear activation is given by:
\begin{equation}
    \bK^{\infty}(\bx_{i}, \bx_{j}) = \bx_{i}^{\textup{T}}\bx_{j} \Big( \pi - \textup{arccos}(\frac{\bx_{i}^{\textup{T}}\bx_{j}}{\|\bx_{i}\|~\|\bx_{j}\|})\Big)/2\pi,
\end{equation}
where $\bK^{\infty}$ is the Neural Tangent Kernel matrix and $\bx_{i} \in \mathbb{R}^{d}$.
\label{def:kernel-matrix}
\end{definition}
We now derive the difference between costs for generated SCFE recourses when the underlying model is \vanilla vs. \robust wide neural network model.
\begin{theorem} (Cost difference for SCFE for wide neural network) For an NTK model with weights $\bw^{\textup{NTK}}_{\textup{NR}}$ and $\bw^{\textup{NTK}}_{\textup{R}}$ for the \vanilla and \robust models, the cost difference between the recourses generated for sample $\bx$ is bounded as:
    \begin{equation}
         \|\zeta_{\textup{NR}}\|_2 - \|\zeta_{\textup{R}}\|_2 \leq \Big|~\frac{2}{\textup{H}(\|\bar{\bw}^{\textup{NTK}}_{\textup{NR}}\|_{2},\|\bar{\bw}^{\textup{NTK}}_{\textup{R}}\|_{2})}~\Big|,
         \label{eq:proof-non-linear-cost-sketch}
    \end{equation}
    where $\textup{H}(\cdot, \cdot)$ denotes the harmonic mean, $\bar{\bw}^{\textup{NTK}}_{\textup{NR}}{=}\nabla_{\bx}\bK^{\infty}(\bx, \bX)\bw^{\textup{NTK}}_{\textup{NR}}$, $\bK^{\infty}$ is the NTK associated with the wide neural network model, $\bar{\bw}^{\textup{NTK}}_{\textup{R}}{=}\nabla_{\bx}\bK^{\infty}(\bx, \bX_{\textup{R}})\bw^{\textup{NTK}}_{\textup{R}}$, $\bw^{\textup{NTK}}_{\textup{NR}}{=}(\bK^{\infty}(\bX, \bX){+}\beta\bI_{n})^{-1}\bY$, $\bw^{\textup{NTK}}_{\textup{R}}{=}(\bK^{\infty}(\bX_{\textup{R}}, \bX_{\textup{R}}) + \beta\bI_{n})^{-1}\bY$, $\beta$ is the bias of the NTK model, $(\bX, \bX_{\textup{R}})$ are the training samples for the \vanilla and \robust models, and $\bY$ are the labels of the training samples.
    \label{thm:cost-bound-non-linear-sketch}
\end{theorem}
\begin{hproof}
    
    The proof follows from Def.~\ref{def:kernel-matrix}, where we use data processing, Taylor expansion, and triangle inequalities to bound the difference between costs of recourses output by SCFE for \vanilla vs. \robust wide neural network models. See Appendix~\ref{sec:proof-thm2-full} for the complete proof.
\end{hproof}
\textit{Implications:} The proposed bounds imply that the difference in costs is bounded by the harmonic mean of the NTK models weights of non-robust and robust models, \textit{i.e.,} the bound is tighter for larger harmonic means, and vice-versa. In particular, the norm of the weight of the \vanilla and \robust NTK model is large if the gradient of NTK associated with the respective model is large.

\paragraph{Cost Analysis of recourses generated using C-CHVAE method.} We extend our analysis of the cost difference for recourses generated using manifold-based methods for \vanilla and \robust models. In particular, we leverage \chvae that leverages variational autoencoders to generate counterfactuals. For a fair comparison, we assume that both models use the same encoder $\cI_{\gamma}$ and decoder $\cG_{\theta}$ networks for learning the latent space of the given input space $\cX$.

\begin{definition} (\citet{bora2017compressed})
An encoder model $\cI$ is $L$-Lipschitz if $~\forall \bz_{1},\bz_{2} \in \cZ$, we have:
\begin{equation}
    \|\cI(\bz_{1}) - \cI(\bz_{2})\|_{p} \leq L \|\bz_{1}-\bz_{2}\|_{p}.
    \label{eq:generative}
\end{equation}
\label{def:generative}
\end{definition}
Next, we derive the bounds of the cost difference of recourses generated for \vanilla and \robust models using Eqn.~\ref{eq:generative}.
\begin{theorem} (Cost difference for C-CHVAE)
    Let $\bz_{\textup{NR}}'$ and $\bz_{\textup{R}}'$ be the solution returned by the \chvae recourse method by sampling from $\ell_{p}$-norm ball in the latent space using an L$_{G}$-Lipschitz decoder $\cG(\cdot)$ for a \vanilla and \robust model. By definition of the recourse method, let $\bx_{\textup{NR}}'{=}\cG(\bz_{\textup{NR}}'){=}\bx+\zeta_{\textup{NR}}$ and $\bx_{\textup{R}}'{=}\cG(\bz_{\textup{R}}'){=}\bx+\zeta_{\textup{R}}$ be the respective recourses in the input space whose difference can then be bounded as:
    \begin{equation}
         \|\zeta_{\textup{NR}}\|_2 - \|\zeta_{\textup{R}}\|_2 \leq \Big|~L_{G}(r_{\textup{R}}+r_{\textup{NR}})~\Big|,
         \label{eq:proof_cost_difference}
    \end{equation}
    where $r_{\textup{NR}}$ and $r_{\textup{R}}$ be the corresponding radii chosen by the algorithm such that they successfully return a recourse for the \vanilla and \robust model.
    \label{thm:cost-bound-cchvae-sketch}
\end{theorem}
\begin{hproof}
    The proof follows from Def.~\ref{def:generative}, triangle inequality, L-Lipschitzness of the generative model, and the fact that the $\ell_p$-norm of the model’s outputs are known in the latent space. See Appendix~\ref{sec:proof-thm3} for detailed proof.
\end{hproof}
\textit{Implications:} The right term in the Eqn.~\ref{eq:proof_cost_difference} entails that the $\ell_p$-norm of the difference between the generated recourses using \chvae is bounded if, 1) the Lipschitz constant of the decoder is small, and 2) 
the sum of the radii determined by \chvae to successfully generate a recourse, such that they successfully return a recourse for the \vanilla and \robust model, is smaller.

\subsection{Validity Analysis}
\label{sec:validity}
The validity of a recourse $\bx'$ is defined as the probability that it results in the desired outcome~\citep{verma2020counterfactual}, \ie, $\Pr(f(\bx') = 1)$. Below, we analyze the validity of the recourses generated for linear models and using Lemma~\ref{thm:weight-linear-sketch} show that it higher for \vanilla models.

\begin{theorem} (Validity comparison for linear model) For a given instance $\mathbf{x} \in \mathbb{R}^{d}$ and desired target label denoted by unity, let $\bx_{\textup{R}}'$ and $\bx_{\textup{NR}}'$ be the counterfactuals for \robust $f_{\textup{R}}(\bx)$ and \vanilla $f_{\textup{NR}}(\bx)$ models, respectively. Then, $\Pr(f_{\textup{NR}}(\bx_{\textup{NR}}') = 1) \geq \Pr(f_{\textup{R}}(\bx_{\textup{R}}') = 1)$ if $|f_{\textup{NR}}(\bx_{\textup{R}}') - f_{\textup{NR}}(\bx_{\textup{NR}}')| \leq \Delta \norm{\bx_{\textup{R}}'}_{2}$. Here, the validity is denoted by $\Pr(f(\bx') = 1)$, which is the probability that the counterfactual $\bx'$ results in the desired outcome.
\label{thm:valid-linear-sketch}
\end{theorem}
\begin{hproof}
    We derive the conditions under which the probability of obtaining a valid recourse is higher for a \vanilla model compared to its \robust counterpart using Lemma~\ref{thm:weight-linear-sketch}, natural logarithms, data processing, and Cauchy-Schwartz inequalities. See Appendix~\ref{app:validity-linear} for the complete proof.
\end{hproof}
\textit{Implications:} We show that the condition for the validity is dependent on the weight difference $\Delta$ of the models (from Lemma~\ref{thm:weight-linear-sketch}). Formally, the validity of \vanilla models will be greater than or equal to that of \robust models only if the difference between the prediction of the \vanilla model on $\bx_{\textup{NR}}'$ and $\bx_{\textup{R}}'$ is bounded by $\Delta$ times the $\ell_{2}$-norm of $\bx_{\textup{R}}'$.

Next, we bound the weight difference of \vanilla and \robust wide neural networks.

\begin{figure*}[t!]
        \begin{flushleft}
            \footnotesize
            \hspace{2.0cm}German-Credit dataset\hspace{3.2cm}Adult dataset\hspace{4.5cm}COMPAS dataset
        \end{flushleft}
        \centering
        \begin{subfigure}[b]{0.32\textwidth}
            \centering
            \includegraphics[width=\textwidth]{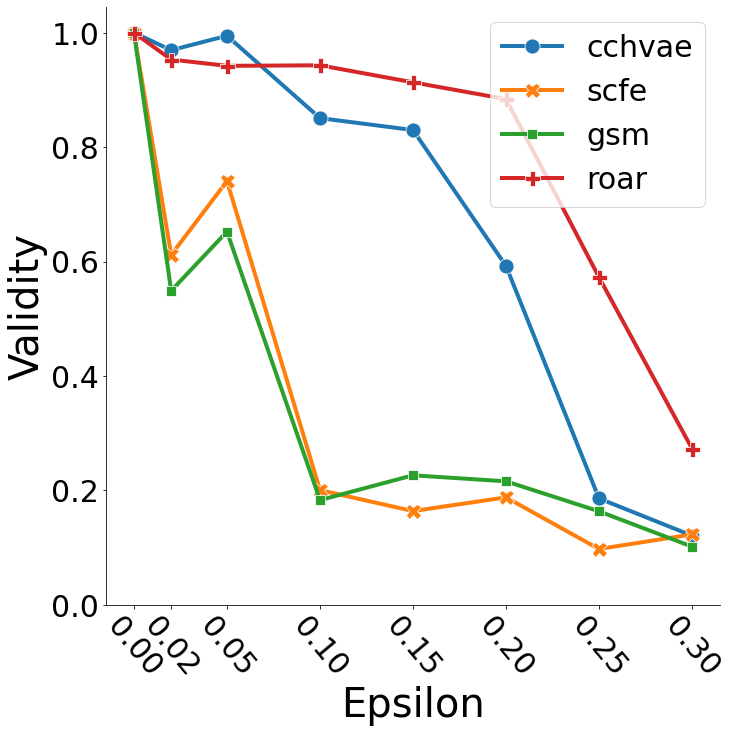}
            \label{fig:cost-adult}
        \end{subfigure}
        \begin{subfigure}[b]{0.32\textwidth}
            \centering
            \includegraphics[width=\textwidth]{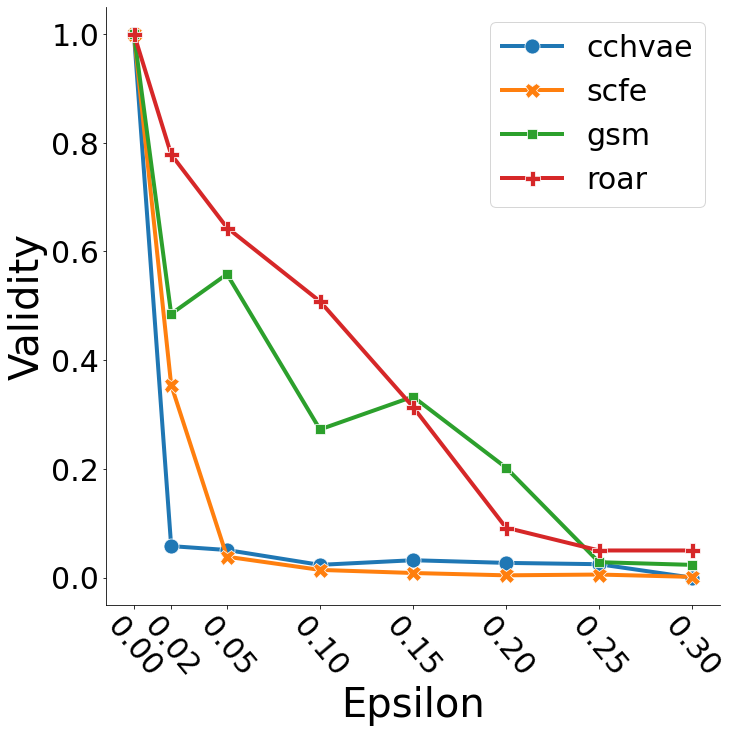}
            \label{fig:cost-compas}
        \end{subfigure}
        \begin{subfigure}[b]{0.32\textwidth}
            \centering
            \includegraphics[width=\textwidth]{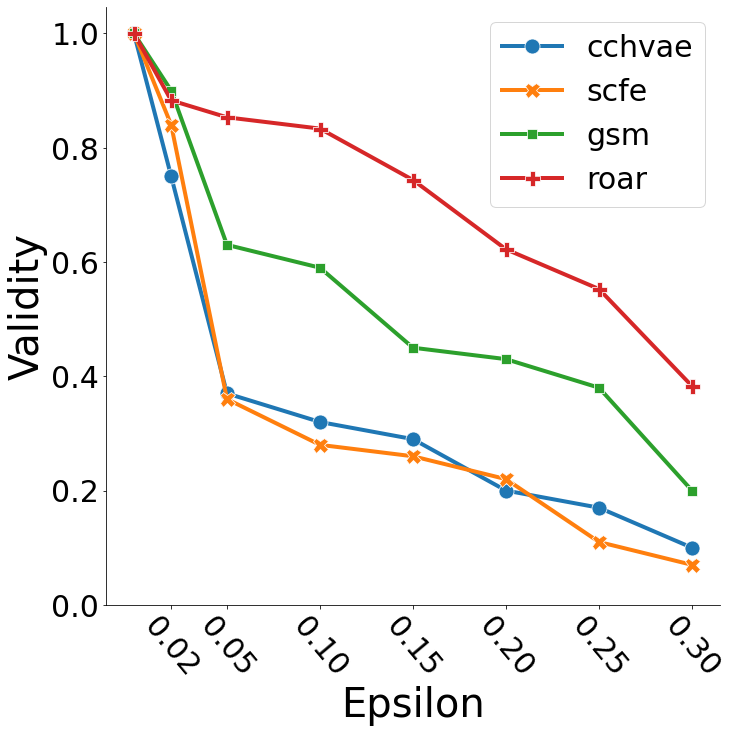}
            \label{fig:validity-adult}
        \end{subfigure}
       
        \caption{Analyzing validity differences between recourses generated using \vanilla and \robust wide neural neural networks for German Credit, Adult, and COMPAS datasets. We find that the validity decreases for increasing values of $\epsilon$. Refer to Appendix~\ref{app:wide_nn} for similar results on larger neural networks.
        }
        \label{fig:all-cost-non-linear}
\end{figure*}

\begin{lemma}(Difference between \vanilla and \robust weights for wide neural network models) For a given NTK model, let $\bw_{\textup{NR}}^{\textup{NTK}}$ and $\bw_{\textup{R}}^{\textup{NTK}}$ be weights of the \vanilla and \robust model. Then, for a wide neural network model with ReLU activations, the difference in the weights can be bounded as:
\begin{equation}
    \|\bw^{\textup{NTK}}_{\textup{NR}}{-}\bw^{\textup{NTK}}_{\textup{R}}\|_2 \leq \Delta_{\textup{K}} \|\bY\|_2
\end{equation}
where $\Delta_{\textup{K}} {=} \|(\bK^{\infty}(\bX, \bX) {+} \beta\bI_{n})^{-1} {-} (\bK^{\infty}(\bX_{\textup{R}}, \bX_{\textup{R}}){+}\\\beta\bI_{n})^{-1}\|_2$, $\bK^{\infty}$ is the kernel matrix for the NTK model defined in Def.~\ref{def:kernel-matrix}, $(\bX, \bX_{\textup{R}})$ are the training samples for the \vanilla and \robust NTK models, $\beta$ is the bias of the ReLU NTK model, and $\bY$ are the labels of the training samples. Subsequently, we show that $\|\bw^{\textup{NTK}}_{\textup{NR}}\|_2 - \Delta_{\textup{K}}\|\bY\|_2 \leq \|\bw^{\textup{NTK}}_{\textup{R}}\|_2 \leq \|\bw^{\textup{NTK}}_{\textup{NR}}\|_2 + \Delta_{\textup{K}}\|\bY\|_2$. 
\label{thm:weight-non-linear-sketch}
\end{lemma}
\begin{hproof}
    We derive the bound for the weight of the adversarially robust NTK model using the closed-form expression for the NTK weights. Using Cauchy-Schwartz and reverse triangle inequality, we prove that the $\ell_2$-norm of the difference between the \vanilla and \robust NTK model weights is upper bounded by the difference between the kernel matrix $\bK^{\infty}$ of the two models. See Appendix~\ref{sec:proof-lemma2-full} for detailed proof.
\end{hproof}
\textit{Implications:} Lemma~\ref{thm:weight-non-linear-sketch} implies that the bound is tight if the generated adversarial samples $\bX_{\text{R}}$ are very close to the original samples, \ie, the degree of robustness of the \robust model is small.

Next, we show that the validity of recourses generated for \vanilla wide neural network models is higher than their \robust counterparts. 

\begin{theorem} (Validity Comparison for wide neural network) For a given instance $\bx \in \mathbb{R}^{d}$ and desired target label denoted by unity, let $\bx_{\textup{R}}'$ and $\bx_{\textup{NR}}'$ be the counterfactuals for \robust $f_{\textup{R}}^{\textup{NTK}}(\bx)$ and \vanilla $f_{\textup{NR}}^{\textup{NTK}}(\bx)$ wide neural network models respectively. Then, $\Pr(f_{\textup{NR}}^{\textup{NTK}}(\bx_{\textup{NR}}') = 1) \geq \Pr(f_{\textup{R}}^{\textup{NTK}}(\bx_{\textup{R}}') = 1)$ if $\norm{(\bK^{\infty}(\bx_{\textup{R}}', \bX_{\textup{R}}) - \bK^{\infty}(\bx_{\textup{NR}}', \bX) )^{\textup{T}} \bw_{\textup{NR}}^{\textup{NTK}}} \leq  \norm{\bK^{\infty}(\bx_{\textup{R}}', \bX_{\textup{R}})^{\textup{T}}} \Delta_{\textup{K}} \|\bY\|_2 $. Here, the validity is denoted by $\Pr(f(\bx') = 1)$, which is the probability that the counterfactual $\bx'$ results in the desired outcome.
\label{thm:valid-nonlinear-sketch}
\end{theorem}
\begin{hproof}
     We extend Theorem~\ref{thm:weight-non-linear-sketch} by deriving an analogous condition for wide neural network models using Lemma~\ref{thm:weight-non-linear-sketch}, natural logarithms, data processing, and Cauchy-Schwartz inequalities. See Appendix~\ref{app:validity-nonlinear} for the complete proof.
\end{hproof}
\textit{Implications:} Our derived conditions show that if the difference between the NTK $\mathbf{K}^{\infty}$ associated with the non-robust and adversarially robust model is bounded (i.e., $\bK^{\infty}(\bx_{\textup{R}}', \bX_{\text{R}})  \approx  \bK^{\infty}(\bx_{\textup{NR}}', \bX)$), then it is likely to have a validity greater than or equal to that of its adversarial robust counterpart. Further, we show that this bound is tighter for smaller $\Delta_{\textup{K}}$, and vice-versa.

\begin{figure*}[t!]
        \begin{flushleft}
            \footnotesize
            \hspace{2.0cm}German-Credit dataset\hspace{3.4cm}Adult dataset\hspace{4.0cm}COMPAS dataset
        \end{flushleft}
        \centering
        \begin{subfigure}[b]{0.32\textwidth}
            \centering
            \includegraphics[width=\textwidth]{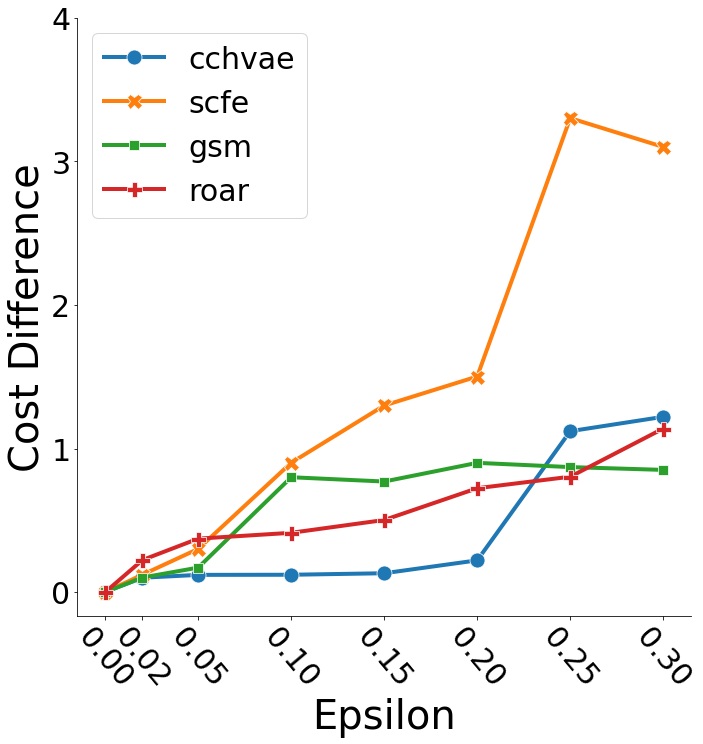}
            \label{fig:cost-adult}
        \end{subfigure}
        \begin{subfigure}[b]{0.32\textwidth}
            \centering
            \includegraphics[width=\textwidth]{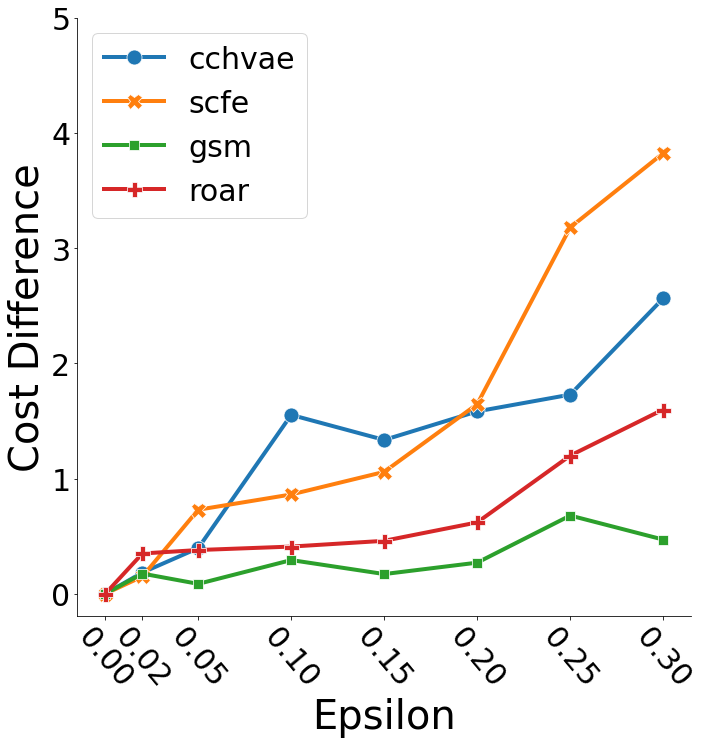}
            \label{fig:cost-compas}
        \end{subfigure}
        \begin{subfigure}[b]{0.32\textwidth}
            \centering
            \includegraphics[width=\textwidth]{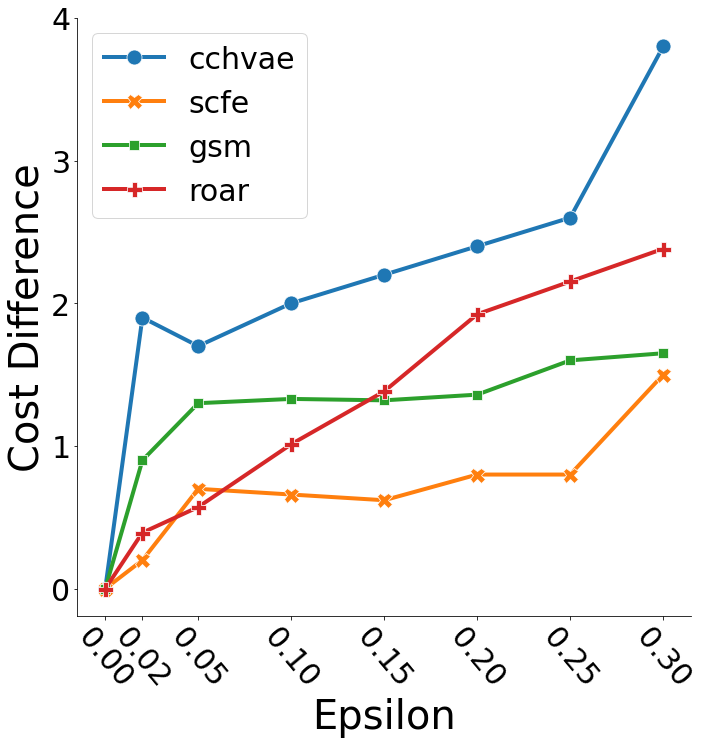}
            \label{fig:validity-adult}
        \end{subfigure}
       
        \caption{Analyzing cost differences between recourses generated using \vanilla and \robust wide neural neural networks for German Credit, Adult, and COMPAS datasets. We find that the cost difference (i.e., $\ell_{2}-$norm) between the recourses generated for \vanilla and \robust models increases for increasing values of $\epsilon$. Refer to Figure~\ref{fig:all-cost-non-linear-large} for similar results on larger neural networks.
        }
        \label{fig:all-cost-non-linear}
\end{figure*}

\begin{figure*}[t!]
        \begin{flushleft}
            \footnotesize
            \hspace{2.0cm}German-Credit dataset\hspace{3.4cm}Adult dataset\hspace{4.0cm}COMPAS dataset
        \end{flushleft}
        \centering
        \begin{subfigure}[b]{0.32\textwidth}
            \centering
            \includegraphics[width=\textwidth]{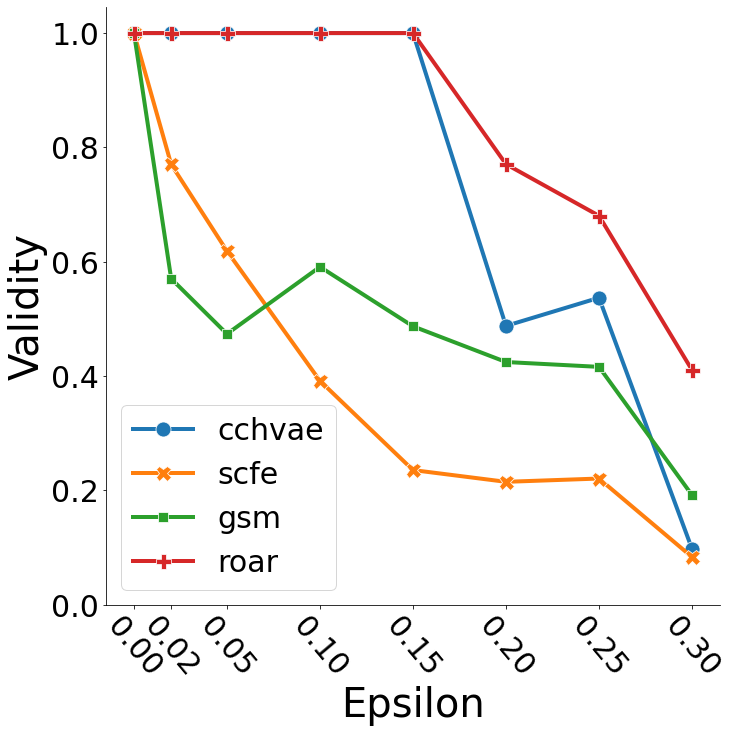}
            \label{fig:cost-adult}
        \end{subfigure}
        \begin{subfigure}[b]{0.32\textwidth}
            \centering
            \includegraphics[width=\textwidth]{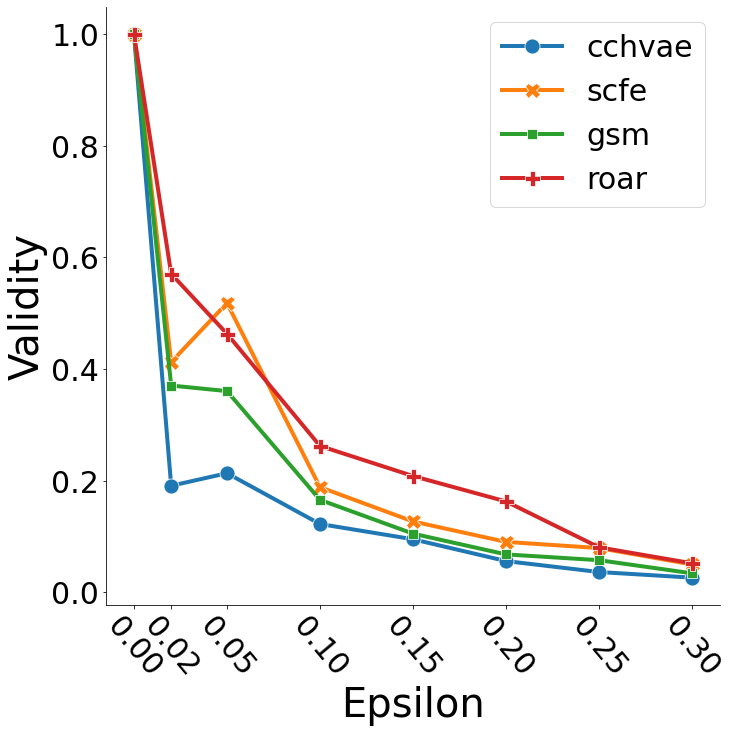}
            \label{fig:cost-compas}
        \end{subfigure}
        \begin{subfigure}[b]{0.32\textwidth}
            \centering
            \includegraphics[width=\textwidth]{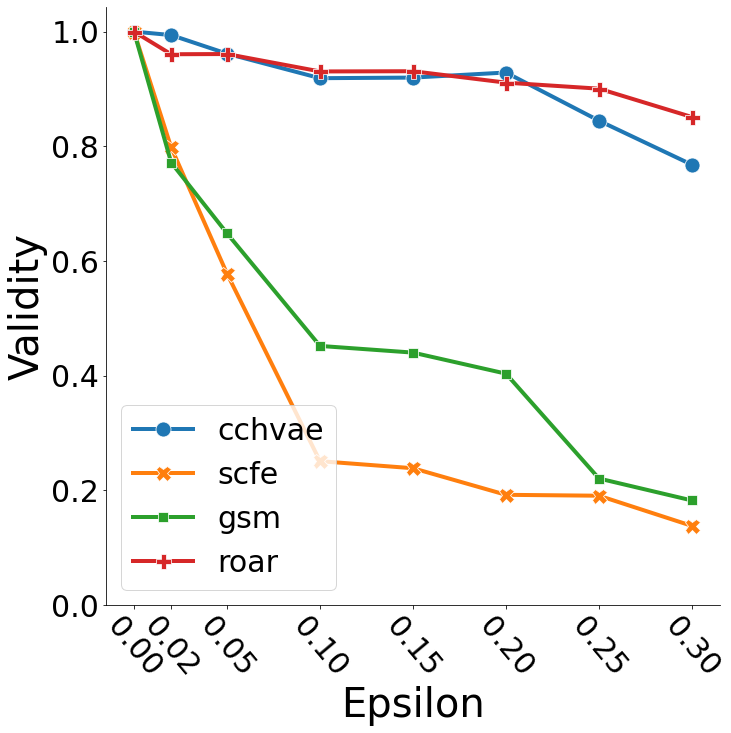}
            \label{fig:validity-adult}
        \end{subfigure}
       
        \caption{Analyzing validity differences between recourses generated using \vanilla and \robust logistic regression for German Credit, Adult, and COMPAS datasets. We find that the validity decreases for increasing values of $\epsilon$. Refer to Appendix~\ref{app:wide_nn} for similar results on larger neural networks.
        }
        \label{fig:all-cost-non-linear}
\end{figure*}

\begin{figure*}[t!]
        \begin{flushleft}
            \footnotesize
            \hspace{2.0cm}German-Credit dataset\hspace{3.4cm}Adult dataset\hspace{4.0cm}COMPAS dataset
        \end{flushleft}
        \centering
        \begin{subfigure}[b]{0.32\textwidth}
            \centering
            \includegraphics[width=\textwidth]{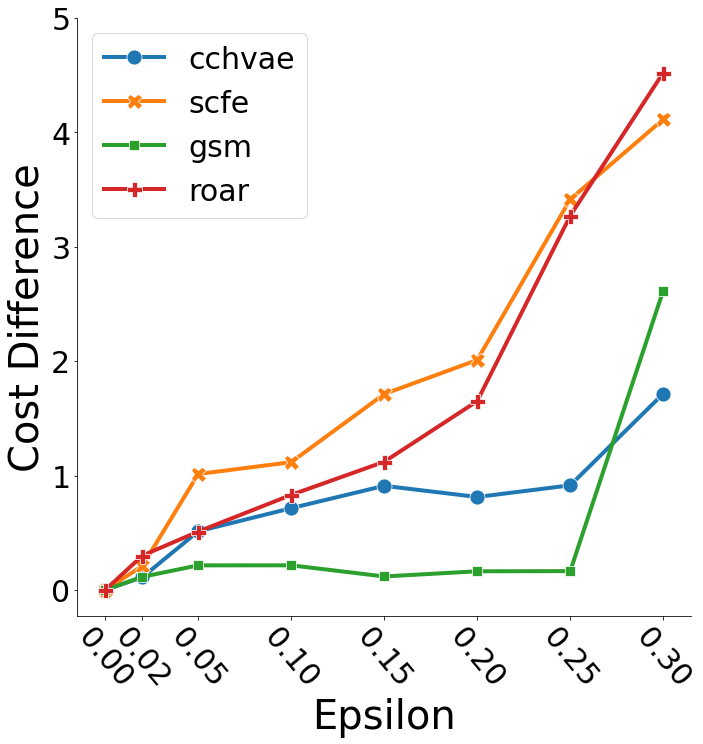}
            \label{fig:cost-adult}
        \end{subfigure}
        \begin{subfigure}[b]{0.32\textwidth}
            \centering
            \includegraphics[width=\textwidth]{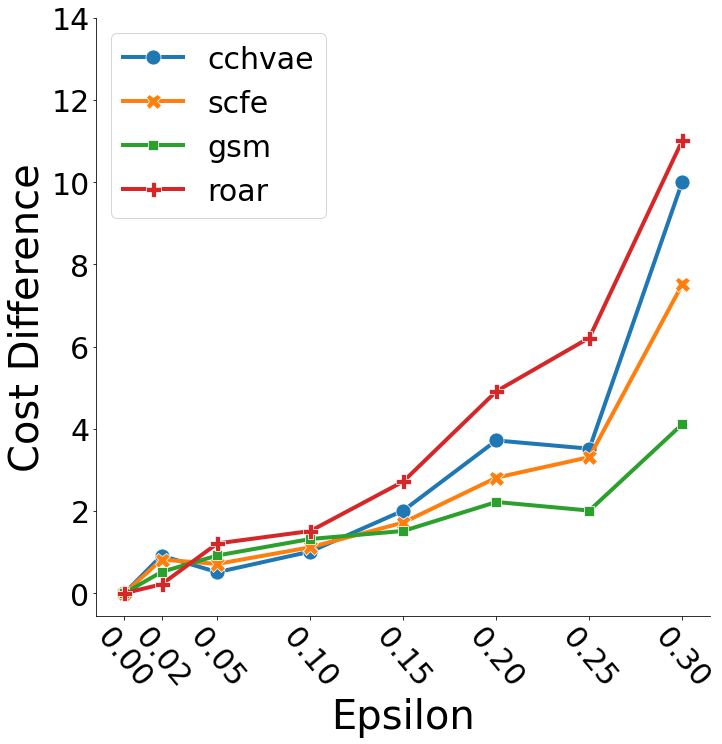}
            \label{fig:cost-compas}
        \end{subfigure}
        \begin{subfigure}[b]{0.32\textwidth}
            \centering
            \includegraphics[width=\textwidth]{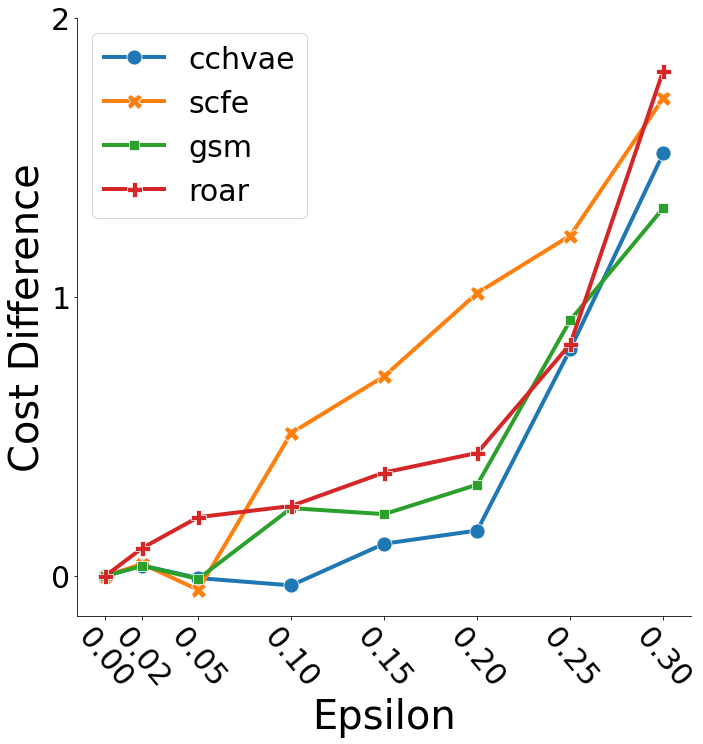}
            \label{fig:cost-adult}
        \end{subfigure}
       
        \caption{Analyzing cost differences between recourses generated using \vanilla and \robust logistic regression for German Credit, Adult, and COMPAS datasets. We find that the cost difference (i.e., $\ell_{2}-$norm) between the recourses generated for \vanilla and \robust models increases for increasing values of $\epsilon$.
        }
        \label{fig:all-cost-non-linear}
\end{figure*}

\begin{figure*}[t!]
        \begin{flushleft}
            \footnotesize
            \hspace{4.2cm}Cost Differences\hspace{7cm}Validity
        \end{flushleft}
        \begin{flushleft}
            \footnotesize
            \hspace{2.3cm}Depth\hspace{3.0cm}Width\hspace{4.5cm}Depth\hspace{3.0cm}Width
        \end{flushleft}
        \centering
        \begin{subfigure}[b]{0.24\textwidth}
            \centering
            \includegraphics[width=\textwidth]{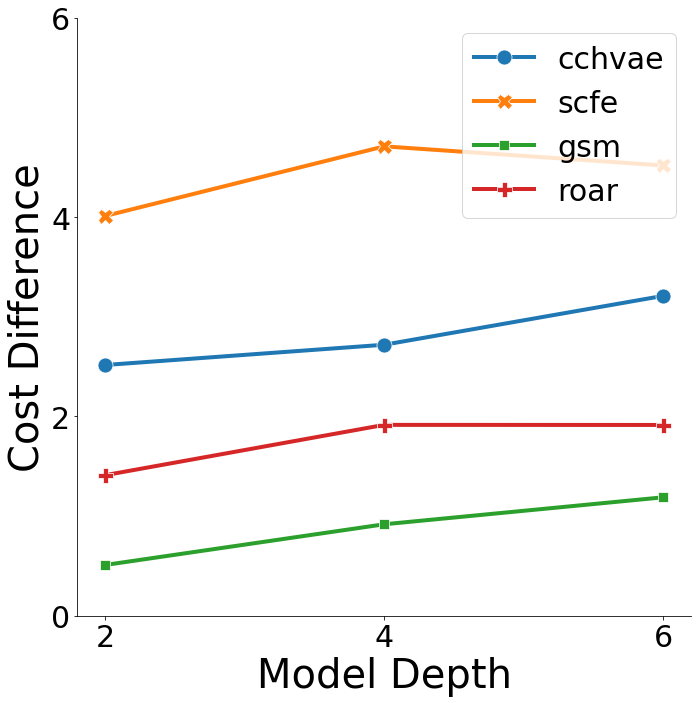}
            \label{fig:cost-adult}
        \end{subfigure}
        \begin{subfigure}[b]{0.24\textwidth}
            \centering
            \includegraphics[width=\textwidth]{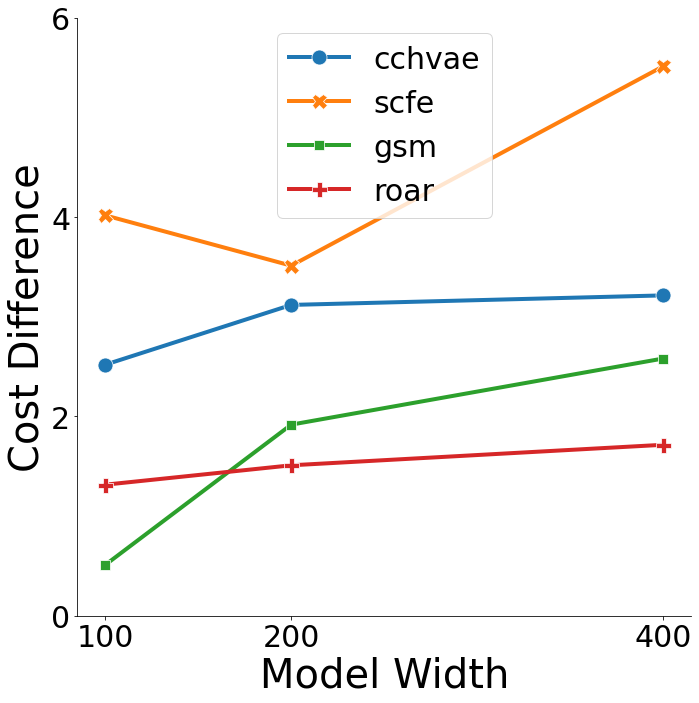}
            \label{fig:cost-compas}
        \end{subfigure}
        \begin{subfigure}[b]{0.24\textwidth}
            \centering
            \includegraphics[width=\textwidth]{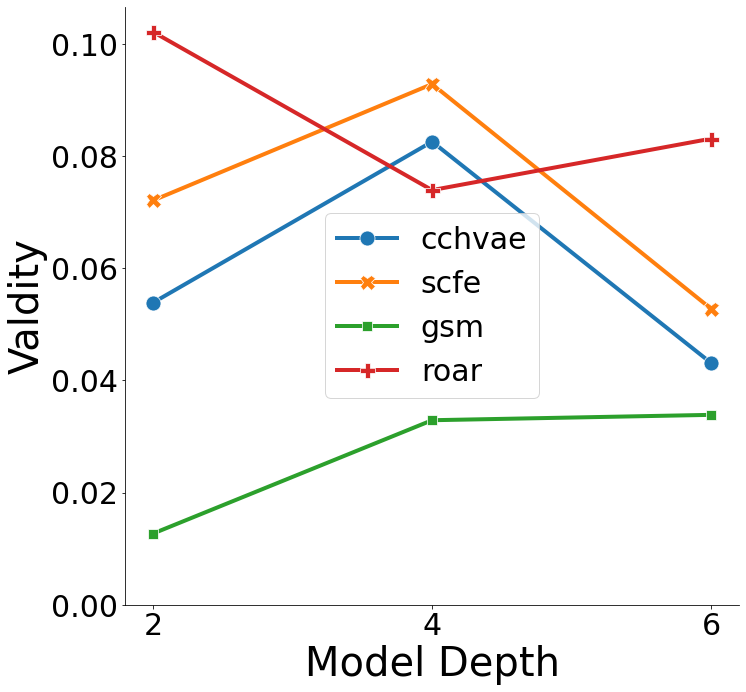}
            \label{fig:validity-adult}
        \end{subfigure}
        \begin{subfigure}[b]{0.24\textwidth}
            \centering
            \includegraphics[width=\textwidth]{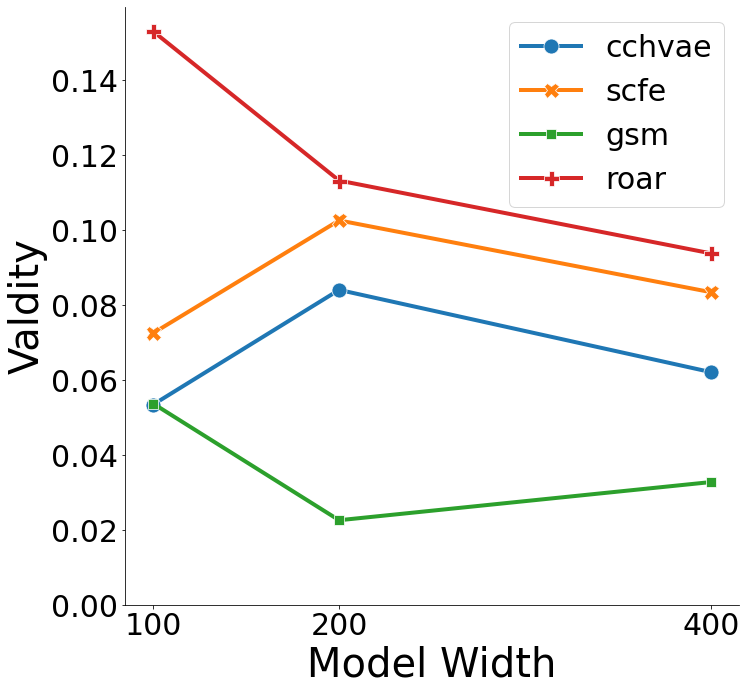}
            \label{fig:validity-compas}
        \end{subfigure}
        \caption{This figure analyzes the cost and validity differences between recourses generated using \vanilla and \robust neural networks trained on the Adult dataset. These differences are examined as the model size increases in terms of depth (defined as the number of hidden layers) and width (defined as the number of nodes in each hidden layer in a neural network of depth=2). Our findings suggest that: i) the cost difference (i.e., $\ell_{2}-$norm) between the recourses generated for \vanilla and \robust models remains consistent even as the model's depth or width increases, and ii) the validity of the recourses remains consistent even as the model's depth or width increases. Here, the \robust model is trained with $\epsilon = 0.3$.
        }
        \label{fig:all-cost-non-linear-large}
\end{figure*}

\begin{figure*}[t!]
        \begin{flushleft}
            \footnotesize
            \hspace{5.0cm}C-CHVAE\hspace{7.1cm}SCFE 
        \end{flushleft}        
        \centering
        
        \begin{subfigure}[b]{0.48\textwidth}
            \centering
            \includegraphics[width=\textwidth]{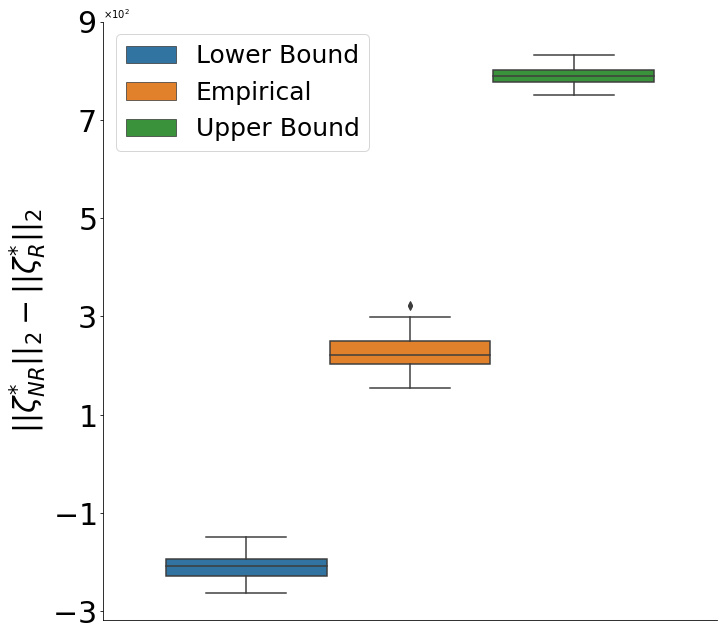}
            \label{fig:theory-scfe}
        \end{subfigure}
        \begin{subfigure}[b]{0.48\textwidth}
            \centering
            \includegraphics[width=\textwidth]{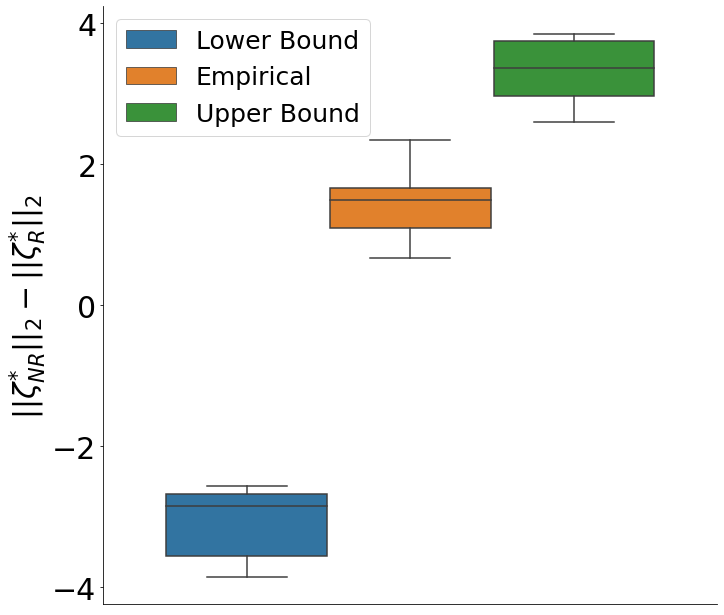}
            \label{fig:theory-cchvae}
        \end{subfigure}
        \caption{
        \textbf{Cost differences (left):} Empirically calculated cost differences (in orange) and our theoretical lower (in blue) and upper (in green) bounds for \chvae and SCFE recourses corresponding to \robust (trained using $\epsilon{=}0.3$) vs. \vanilla linear models trained on the Adult dataset. See Fig.~\ref{fig:theory-validation-bound-nn} for similar bounds for \robust ($\epsilon{=}0.3$) vs. \vanilla neural-network model. \textbf{Validity (right):} Empirical difference between the validity of recourses for \vanilla and \robust linear and neural network model trained on Adult dataset.  Results show no violations of our theoretical bounds.}
        \label{fig:theory-validation-full}
\end{figure*}
\begin{figure*}[t!]
        \centering
        \begin{subfigure}[b]{0.45\textwidth}
            \centering
            \includegraphics[width=\textwidth]{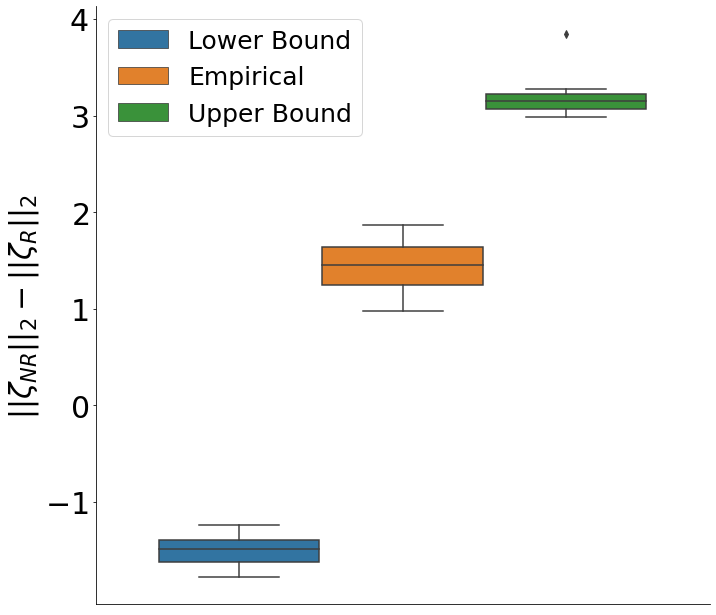}
            \caption{}
            \label{fig:theory-validity-non-linear}
        \end{subfigure}
        \begin{subfigure}[b]{0.45\textwidth}
            \centering
            \includegraphics[width=\textwidth]{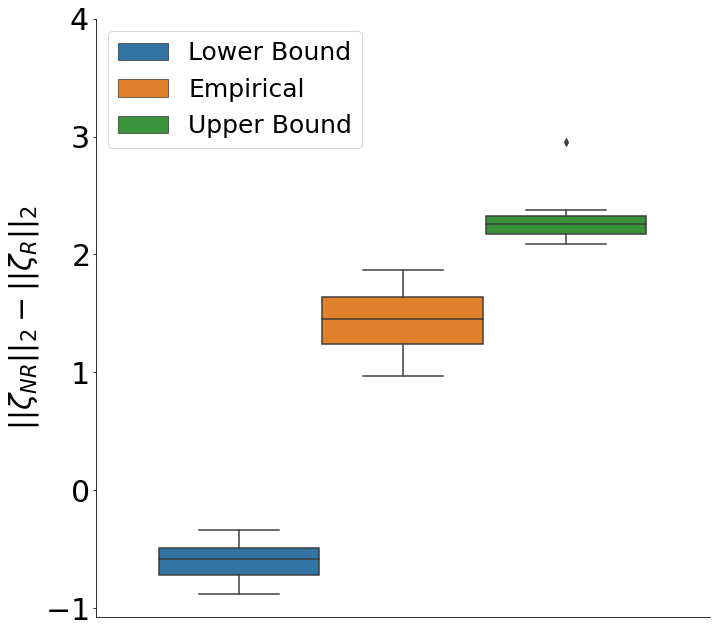}
            \caption{}
            \label{fig:theory-cost}
        \end{subfigure}
        \caption{
        (a) Empirically calculated cost differences (in orange) for the original model and our theoretical lower (in blue) and upper (in green) bounds for SCFE recourses corresponding to \robust (trained using $\epsilon{=}0.1$) vs. \vanilla neural networks corresponding to test samples of the Adult dataset, based on Theorem \ref{thm:cost-bound-non-linear-sketch}. (b) Empirically calculated cost differences (in orange) for the original model and our theoretical lower (in blue) and upper (in green) bounds for SCFE recourses corresponding to \robust (trained using $\epsilon{=}0.2$) vs. \vanilla neural networks corresponding to test samples of the Adult dataset, based on Theorem \ref{thm:cost-bound-non-linear-sketch}.
        }
        \label{fig:theory-validation-bound-nn}
\end{figure*}

\begin{figure}[t!]
        \centering
        \begin{subfigure}[b]{0.36\textwidth}
            \centering
            \includegraphics[width=\textwidth]{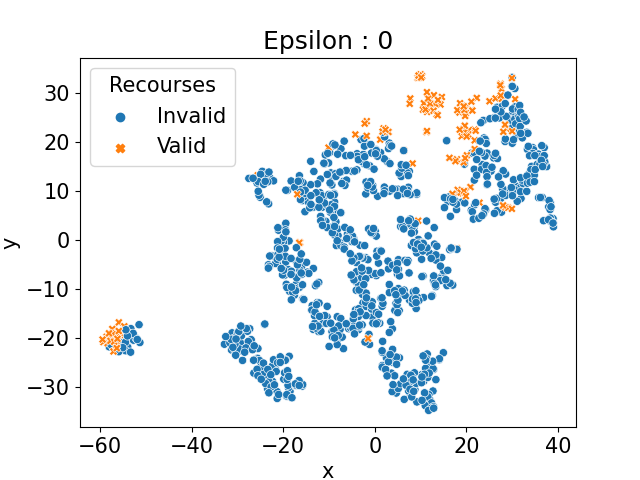}
            \label{fig:adv_accuracy}
        \end{subfigure}
        \begin{subfigure}[b]{0.36\textwidth}
            \centering
            \includegraphics[width=\textwidth]{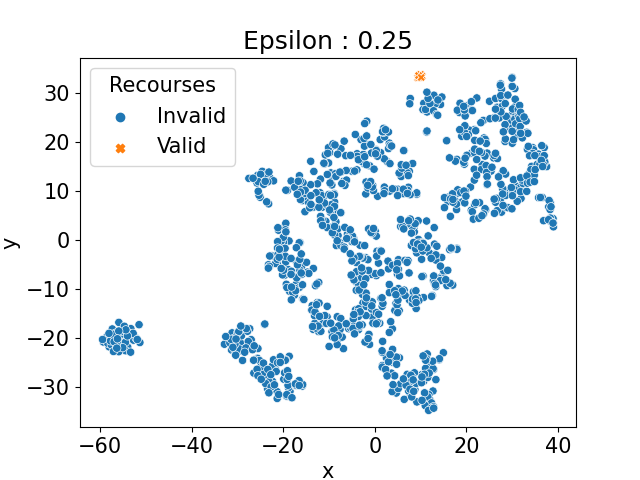}
            \label{fig:viz3}
        \end{subfigure}
        \caption{A t-SNE visualization of the change in availability of valid recourses (orange) for \robust models trained using $\epsilon = [0, 0.25]$, where a \vanilla model is a model trained using $\epsilon=0$. Results are shown for a neural network model trained on the Adult dataset. We observe fewer valid recourses for higher values of $\epsilon$ in this local neighborhood.
        }
        \label{fig:adult-viz}
\end{figure}

%% file: 050experiments.tex
\section{Experimental Evaluation}
\label{sec:expt}

In this section, we empirically analyze the impact of \robust models on the cost and validity of recourses. First, we empirically validate our theoretical bounds on differences between the cost and validity of recourses output by state-of-the-art recourse generation algorithms when the underlying models are \robust vs. \vanilla. Second, we carry out further empirical analysis to assess the differences in cost and validity of the resulting recourses as the degree of the adversarial robustness of the underlying model changes on three real-world datasets.

\subsection{Experimental Setup}
\label{sec:setup}
Here, we describe the datasets, predictive models, \algor generation methods, and the evaluation metrics used in our empirical analysis.

\paragraph{Datasets.} We use three real-world datasets for our experiments: 1) The \textit{German Credit}~\citep{UCI} dataset comprises demographic (age, gender), personal (marital status), and financial (income, credit duration) features from 1000 credit applicants, with each sample labeled as "good" or "bad" depending on their credit risk. The task is to successfully predict if a given individual is a "good" or "bad" customer in terms of associated credit risk. 2) The \textit{Adult}~\citep{yeh2009comparisons} dataset contains demographic (e.g., age, race, and gender), education (degree), employment (occupation, hours-per week), personal (marital status, relationship), and financial (capital gain/loss) features for 48,842 individuals. The task is to predict if an individual’s income exceeds \$50K per year. 3) The \textit{COMPAS}~\citep{jordan15:effect} dataset has criminal records and demographics features for 18,876 defendants who got released on bail at the U.S state courts during 1990-2009. The dataset is designed to train a binary classifier to classify defendants into bail (i.e., unlikely to commit a violent crime if released) vs. no bail (i.e., likely to commit a violent crime).

\paragraph{Predictive models.} We generate recourses for the \vanilla and \robust version of Logistic Regression (linear) and Neural Networks (non-linear) models. We use two linear layers with ReLU activation functions as our predictor and set the number of nodes in the intermediate layers to twice the number of nodes in the input layer, which is the size of the input dimension in each dataset. 

 \paragraph{\AlgoR Methods.} We analyze the cost and validity for recourses generated using four popular classes of recourse generation methods, namely, gradient-based (SCFE), manifold-based (C-CHVAE), random search-based (GSM) methods (see Sec.~\ref{sec:prelim}), and robust methods (ROAR)~\cite{upadhyay2021robust}, when the underlying models are \vanilla and \robust.

\paragraph{Evaluation metrics.} To concretely measure the impact of adversarial robustness on \algor, we analyze the difference between cost and validity metrics for recourses generated using \vanilla and \robust model. To quantify the cost, we measure the average cost incurred to act upon the prescribed recourses across all test-set instances, \ie, $\textup{Cost}(\bx, \bx') = \frac{1}{|\cD_{\textup{test}}   |}\|\bx - \bx'\|_{2}$, where $\bx$ is the input and $\bx'$ is its corresponding recourse. To measure validity, we compute the probability of the generated recourse resulting in the desired outcome, \ie, $\textup{Validity}(\bx, \bx') = \frac{|\{\bx': f(\bx') = 1 \,\, \cap \, \, \bx' = g(\bx, f) \}|}{|\cD_{\textup{test}}|}$, where $g(x, f)$ returns recourses for input $\bx$ and predictive model $f$.

\paragraph{Implementation details.} We train \vanilla and \robust predictive models from two popular model classes (logistic regression and neural networks) for all three datasets. In the case of \robust models, we adopt the commonly used min-max optimization objective for adversarial training using varying degree of robustness, i.e., $\epsilon \in \{ 0, 0.02, 0.05, 0.10, 0.15, 0.20, 0.25, 0.3\}$. Note that the model trained with $\epsilon{=}0$ is the \vanilla model.

\subsection{Empirical Analysis}
\label{sec:results}
Next, we describe the experiments we carried out to understand the impact of adversarial robustness of predictive models on \algor. More specifically, we discuss the (1) empirical verification of our theoretical bounds, (2) empirical analysis of the differences between the costs of recourses when the underlying model is \vanilla vs. \robust, and (3) empirical analysis to compare the validity of the recourses corresponding to \vanilla vs. \robust models.

\paragraph{Empirical Verification of Theoretical Bounds.} We empirically validate our theoretical findings from Section ~\ref{sec:theory} on real-world datasets. 
In particular, we first estimate the empirical bounds (RHS of Theorems~\ref{thm:cost-bound-linear-sketch},\ref{thm:cost-bound-cchvae-sketch}) for each instance in the test set by plugging the corresponding values of the parameters in the theorems and compare them with the empirical estimates of the cost differences between recourses generated using gradient-based and manifold-based  recourse methods (LHS of Theorems~\ref{thm:cost-bound-linear-sketch},\ref{thm:cost-bound-cchvae-sketch}). Figure~\ref{fig:theory-validation-full} shows the results obtained from the aforementioned analysis of cost differences. We observe that our bounds are tight, and the empirical estimates fall well within our theoretical bounds. A similar trend is observed for Theorem~\ref{thm:cost-bound-non-linear-sketch} which is for the case of non-linear models, shown in Figure ~\ref{fig:theory-validation-bound-nn}. For the case of theoretical bounds for validity analysis in Theorem~\ref{thm:valid-linear-sketch}, we observe that the validity of the \vanilla model (denoted by $\Pr(f_{\text{NR}}(x) = 1)$ in Theorem~\ref{thm:valid-linear-sketch}) was higher than the validity of the \robust model for all the test samples  following the condition in Theorem ~\ref{thm:valid-linear-sketch} ($>$ 90\% samples) for a large number of training iterations used for training \robust models with $\epsilon \in \{ 0, 0.02, 0.05, 0.1, 0.15, 0.2, 0.25, 0.3\}$, shown Figure ~\ref{fig:theory-validation-full}.

\paragraph{Cost Analysis.} To analyze the impact of adversarial robustness on the cost of recourses, we compute the difference between the cost for obtaining a recourse using \vanilla and \robust model and plotted this difference for varying degrees of robustness $\epsilon$. Results in Figure~\ref{fig:all-cost-non-linear} show a significant increase in costs to find \algor for \robust neural network models with increasing degrees of robustness for all the datasets. In addition, the recourse cost for \robust model is always higher than that of \vanilla model (see appendix Figure~\ref{fig:cost-analysis-all} for similar trends for logistic regression model). Further, we observe a relatively smoother increasing trend for SCFE cost differences compared to others. We attribute this trend to the stochasticity present in \chvae and GSM. We find a higher cost difference in SCFE for most datasets, which could result from the larger sample size used in \chvae and GSM. Further, we observe a similar trend in cost differences on increasing the number of iterations to find a recourse.

\paragraph{Validity Analysis.} To analyze the impact of adversarial robustness on the validity of recourses, we compute the fraction of recourses resulting in the desired outcome, generated using a \vanilla and \robust model under resource constraints, and plot it against varying degrees of robustness $\epsilon$. Results in Figure~\ref{fig:theory-validation-full} show a strong impact of adversarial training on the validity for logistic regression and neural network models trained on three real-world datasets (also see Appendix~\ref{app:results}). On average, we observe that the validity drops to zero for models adversarially trained with $\epsilon > 0.2$. To shed more light on this observation, we use t-SNE visualization~\citep{van2008visualizing} -- a non-linear dimensionality reduction technique -- to visualize test samples in the dataset to two-dimensional space. In Figure~\ref{fig:adult-viz}, we observe a gradual decline in the number of valid recourses around a local neighborhood with an increasing degree of robustness $\epsilon$. The decline in the number of valid recourses suggests that multiple recourses in the neighborhood of the input sample are classified to the same class as the input for higher $\epsilon$, supporting our hypothesis that \robust models severely impact the validity of recourses and make the recourse search computationally expensive.

%% file: 060conclusions.tex
\section{Conclusion}
In this work, we theoretically and empirically analyzed the impact of \robust models on actionable explanations. We theoretically bounded the cost differences between recourses output by state-of-the-art techniques when the underlying models are \robust vs. \vanilla. We also bounded the validity differences between the recourses corresponding to \robust vs. \vanilla models. Further, we empirically validated our theoretical results using three real-world datasets and two popular classes of predictive models. Our theoretical and empirical analyses demonstrate that \robust models significantly increase the cost and decrease the validity of the resulting recourses, thereby highlighting the inherent trade-offs between achieving adversarial robustness in predictive models and providing reliable algorithmic recourses. Our work also paves the way for several interesting future research directions at the intersection of algorithmic recourse and adversarial robustness in predictive models. For instance, given the aforementioned trade-offs, it would be interesting to develop novel techniques that enable end users to navigate these trade-offs based on their personal preferences, \eg, an end user may choose to sacrifice the adversarial robustness of the underlying model to secure lower cost recourses.

%% file: 070appendix.tex
\setcounter{lemma}{0}
\setcounter{theorem}{0}

\onecolumn 

\section{Proof for Theorems in Section~\ref{sec:theory}}
Here, we provide detailed proofs of the Lemmas and Theorems defined in Section~\ref{sec:theory}.
\subsection{Proof for Lemma~\ref{thm:weight-linear-sketch}}
\label{sec:proof-lemma1-full}
\begin{lemma}(Difference between \vanilla and \robust model weights for linear model) For a given instance $\bx$, let $\bw_{\textup{NR}}$ and $\bw_{\textup{R}}$ be weights of the \vanilla and \robust model. Then, for a normalized Lipschitz activation function $\sigma(\cdot)$, the difference in the weights can be bounded as:
\begin{equation}
    \|\mathbf{w}_{\textup{NR}} - \mathbf{w}_{\textup{R}}\|_2 \leq \Delta
\end{equation}
where $\Delta = N\eta ( y\|\mathbf{x}^{T}\|_{2} + \epsilon \sqrt{d})$, $\eta$ is the learning rate, $\epsilon$ is the $\ell_{2}$-norm perturbation ball, $y$ is the label for $\bx$, $N$ is the total number of training epochs, and $d$ is the dimension of the input features. Subsequently, we show that $\|\mathbf{w}_{\textup{NR}}\|_2 - \Delta \leq \|\mathbf{w}_{\textup{R}}\|_2 \leq \|\mathbf{w}_{\textup{NR}}\|_2 + \Delta$.
\label{thm:weight-linear-full}
\end{lemma}

\begin{proof}
Without loss of generality, we consider the case of binary classification which uses the binary cross entropy or logistic loss. Let us denote the baseline and robust models as $f_{\text{NR}}(\mathbf{x}){=}\mathbf{w}^{\text{T}}_{\text{NR}}\mathbf{x}$ and $f_{\text{R}}(\mathbf{x}){=}\mathbf{w}^{\text{T}}_{\text{R}}\mathbf{x}$, where we have removed the bias term for simplicity. We consider the class label as $y \in \{+1, -1\}$, and loss function $\mathcal{L}(f(\mathbf{x})) = \text{log}(1 + \exp(-f(\mathbf{x})))$. Note that an adversarially robust model $f_{\text{R}}(\mathbf{x})$ is commonly trained using a min-max objective, where the inner maximization problem is given by:
\begin{equation}
    \max_{\|\delta\| \leq \epsilon} \mathcal{L}(\mathbf{w}^{\text{T}}_{\text{R}} (\mathbf{x}+\delta), y),
\end{equation}

where $\delta$ is the adversarial perturbation added to a given sample $\mathbf{x}$ and $\epsilon$ denotes the the perturbation norm ball around $\mathbf{x}$. Since our loss function is monotonic decreasing, the maximization of the loss function applied to a scalar is equivalent to just minimizing the scalar quantity itself, i.e.,
\begin{equation}
\max_{\|\delta\| \leq \epsilon} \mathcal{L} \left(y \cdot (\mathbf{w}^{\text{T}}_{\text{R}}(\mathbf{x}+\delta)) \right) =
\mathcal{L}\left( \min_{\|\delta\| \leq \epsilon}  y \cdot (\mathbf{w}^{\text{T}}_{\text{R}}(\mathbf{x}+\delta)) \right) \\
= \mathcal{L}\left(y\cdot(\mathbf{w}^{\text{T}}_{\text{R}}\mathbf{x}) + \min_{\|\delta\| \leq \epsilon} y \cdot \mathbf{w}^{\text{T}}_{\text{R}}\delta  \right)
\end{equation}

The optimal solution to $\min_{\|\delta\| \leq \epsilon} y \cdot \mathbf{w}^{\text{T}}_{\text{R}}\delta$ is given by $-\epsilon \|\mathbf{w}^{\text{T}}_{\text{R}}\|_{1}$~\citep{madryzico}. Therefore, instead of solving the min-max problem for an adversarially robust model, we can convert it to a pure minimization problem, i.e., 
\begin{equation}
    \min_{\mathbf{w}_{\text{R}}} \mathcal{L} \left(y \cdot (\mathbf{w}^{\text{T}}_{\text{R}}\mathbf{x}) - \epsilon \|\mathbf{w}_{\text{R}}\|_{1} \right )
\end{equation}

Correspondingly, the minimization objective for a baseline model is given by $\min_{\bw_{\text{NR}}} \cL \left(y \cdot (\bw^{\text{T}}_{\text{NR}}\bx)\right )$. Looking into the training dynamics under gradient descent, we can define the weights at epoch `t' for a baseline and robust model as a function of the Jacobian of the loss function with respect to their corresponding weights, i.e., 
\begin{equation}
    \frac{\mathbf{w}_{\text{NR}} - \mathbf{w}_{0}}{\eta} = \nabla_{\mathbf{w}_{\text{NR}}}\mathcal{L}\left(y.f_{\text{NR}}(\mathbf{x})\right ),
    \label{eq:sgd-base}
\end{equation}
\begin{equation}
    \frac{\mathbf{w}_{\text{R}} - \mathbf{w}_{0}}{\eta} = \nabla_{\mathbf{w}_{\text{R}}} \mathcal{L}\left(y.f_{\text{R}}(\mathbf{x}) - \epsilon \|\mathbf{w}_{\text{R}}\|_{1} \right ),
    \label{eq:sgd-robust}
\end{equation}
where $\eta$ is the learning rate of the gradient descent optimizer, $\mathbf{w}_{0}$ is the weight initialization of both models.
\begin{flalign*}
\nabla_{\mathbf{w}_{\text{NR}}}\mathcal{L}\left(y.f_{\text{NR}}(\mathbf{x})\right ) &= -\frac{\exp(-y.f_{\text{NR}}(\mathbf{x}))}{1 + \exp(-y.f_{\text{NR}}(\mathbf{x}))}y.\mathbf{x}^{\text{T}} \\
\nabla_{\mathbf{w}_{\text{R}}} \mathcal{L}\left(y.f_{\text{R}}(\mathbf{x}) - \epsilon \|\mathbf{w}_{\text{R}}\|_{1} \right ) &= - \frac{\exp(-y.f_{\text{R}}(\mathbf{x}) + \epsilon \|\mathbf{w}_{\text{R}}\|_{1})}{1 + \exp(-y.f_{\text{R}}(\mathbf{x}) + \epsilon \|\mathbf{w}_{\text{R}}\|_{1})}(-y.\mathbf{x}^{\text{T}} + \epsilon . \text{sign}(\mathbf{w}_{\text{R}})),
\end{flalign*}
where $\text{sign}(\mathbf{x})$ return +1, -1, 0 for $x>0$, $x<0$, $x=0$ respectively and $\sigma(\mathbf{x}){=}\frac{1}{1 + \exp(-\mathbf{x})}$ is the sigmoid function. Let us denote the weights of the baseline and robust model at the $n$-th iteration as $\mathbf{w}_{\text{NR}}^{n}$ and $\mathbf{w}_{\text{R}}^{n}$, respectively. Hence, we can define the $n$-th step of the gradient-descent for both models as:
\begin{eqnarray}
    \mathbf{w}_{\text{NR}}^{n} - \mathbf{w}_{\text{NR}}^{n-1} &= \eta \nabla_{\mathbf{w}_{\text{NR}}^{n-1}} \mathcal{L}_{\text{NR}}(\cdot)\\
    \mathbf{w}_{\text{R}}^{n} - \mathbf{w}_{\text{R}}^{n-1} &= \eta \nabla_{\mathbf{w}_{\text{R}}^{n-1}}\mathcal{L}_{\text{R}}(\cdot),
\end{eqnarray}
where $\eta$ is the learning rate of the gradient descent optimizer. Taking $n=1$, we get:
\begin{eqnarray}
    \mathbf{w}_{\text{NR}}^{1} - \mathbf{w}_{\text{NR}}^{0} &= \eta \nabla_{\mathbf{w}_{\text{NR}}^{0}} \mathcal{L}_{\text{NR}}(\cdot) \\
    \mathbf{w}_{\text{R}}^{1} - \mathbf{w}_{\text{R}}^{0} &= \eta \nabla_{\mathbf{w}_{\text{R}}^{0}} \mathcal{L}_{\text{R}}(\cdot),
\end{eqnarray}
where $\mathbf{w}_{\text{NR}}^{0}$ and $\mathbf{w}_{\text{R}}^{0}$ are the same initial weights for the \vanilla and \robust models. Subtracting both equations, we get:
\begin{equation}
    \frac{\mathbf{w}_{\text{NR}}^{1} - \mathbf{w}_{\text{R}}^{1}}{\eta} = \nabla_{\mathbf{w}_{\text{NR}}^{0}}\mathcal{L}_{\text{NR}}(\cdot) - \nabla_{\mathbf{w}_{\text{R}}^{0}}\mathcal{L}_{\text{R}}(\cdot)
    \label{eq:23}
\end{equation}
Similarly, for $n=2$ and using Equation~\ref{eq:23}, we get the following relation:
\begin{flalign*}
    \frac{\mathbf{w}_{\text{NR}}^{2} - \mathbf{w}_{\text{R}}^{2}}{\eta} - \Big( \frac{\mathbf{w}_{\text{NR}}^{1} - \mathbf{w}_{\text{R}}^{1}}{\eta} \Big) &= \nabla_{\mathbf{w}_{\text{NR}}^{1}}\mathcal{L}_{\text{NR}}(\cdot) - \nabla_{\mathbf{w}_{\text{R}}^{1}}\mathcal{L}_{\text{R}}(\cdot)\\
    \frac{\mathbf{w}_{\text{NR}}^{2} - \mathbf{w}_{\text{R}}^{2}}{\eta} &= \nabla_{\mathbf{w}_{\text{NR}}^{0}}\mathcal{L}_{\text{NR}}(\cdot) + \nabla_{\mathbf{w}_{\text{NR}}^{1}}\mathcal{L}_{\text{NR}}(\cdot) -
    \nabla_{\mathbf{w}_{\text{R}}^{0}}\mathcal{L}_{\text{R}}(\cdot) -
    \nabla_{\mathbf{w}_{\text{R}}^{1}}\mathcal{L}_{\text{R}}(\cdot)
\end{flalign*}
Using the above equations, we can now write the difference between the weights of the \vanilla and \robust models at the $n$-th iteration as:
\begin{flalign*}
    \frac{\mathbf{w}_{\text{NR}}^{n} - \mathbf{w}_{\text{R}}^{n}}{\eta} &= \sum_{i=0}^{n-1} \nabla_{\mathbf{w}_{\text{NR}}^{i}}\mathcal{L}_{\text{NR}}(\cdot) - \sum_{i=0}^{n-1} \nabla_{\mathbf{w}_{\text{R}}^{i}}\mathcal{L}_{\text{R}}(\cdot)\\
    &= \sum_{i=0}^{n-1} (\sigma(y.f_{\text{NR}}^{i}(\mathbf{x})) - 1)y.\mathbf{x}^{T} - \sum_{i=0}^{n-1} (\sigma(y.f_{\text{R}}^{i}(\mathbf{x}) - \epsilon||\mathbf{w}_{\text{R}}^{i}||_{1}) - 1)(y.\mathbf{x}^{T} - \epsilon~\text{sign}(\mathbf{w}_{\text{R}}^{i}))\\
    &= \sum_{i=0}^{n-1} \sigma(y.f_{\text{NR}}^{i}(\mathbf{x}))y.\mathbf{x}^{T} + \sum_{i=0}^{n-1} \Big( \sigma(y.f_{\text{R}}^{i}(\mathbf{x}) - \epsilon||\mathbf{w}_{\text{R}}^{i}||_{1})(\epsilon~\text{sign}(\mathbf{w}_{\text{R}}^{i}) - y.\mathbf{x}^{T}) - \epsilon~\text{sign}(\mathbf{w}_{\text{R}}^{i}) \Big)\\
    &\leq \sum_{i=0}^{n-1} \sigma(y.f_{\text{NR}}^{i}(\mathbf{x}))y.\mathbf{x}^{T} + \sum_{i=0}^{n-1} \Big( \sigma(y.f_{\text{R}}^{i}(\mathbf{x}))(\epsilon~\text{sign}(\mathbf{w}_{\text{R}}^{i}) - y.\mathbf{x}^{T}) - \epsilon~\text{sign}(\mathbf{w}_{\text{R}}^{i}) \Big) \tag{Using $\sigma(a-b)\leq\sigma(a)$~~for $b>0$}\\
    &\leq \sum_{i=0}^{n-1} \Big(\sigma(y.f_{\text{NR}}^{i}(\mathbf{x}) - \sigma(y.f_{\text{R}}^{i}(\mathbf{x}))\Big)y.\mathbf{x}^{T} + \sum_{i=0}^{n-1} \Big( \sigma(y.f_{\text{R}}^{i}(\mathbf{x})) - 1\Big) \epsilon~\text{sign}(\mathbf{w}_{\text{R}}^{i})
\end{flalign*}

Using $\ell_{2}$-norm on both sides, we get:
\begin{flalign*}
    \frac{1}{\eta}\|\mathbf{w}_{\text{NR}}^{n} - \mathbf{w}_{\text{R}}^{n}\|_{2} &\leq \|\sum_{i=0}^{n-1} \Big(\sigma(y.f_{\text{NR}}^{i}(\mathbf{x}) - \sigma(y.f_{\text{R}}^{i}(\mathbf{x}))\Big)y.\mathbf{x}^{T} + \sum_{i=0}^{n-1} \Big( \sigma(y.f_{\text{R}}^{i}(\mathbf{x})) - 1\Big) \epsilon~\text{sign}(\mathbf{w}_{\text{R}}^{i}) \|_{2}\\
    &\leq \|\sum_{i=0}^{n-1} \Big(\sigma(y.f_{\text{NR}}^{i}(\mathbf{x}) - \sigma(y.f_{\text{R}}^{i}(\mathbf{x}))\Big)y.\mathbf{x}^{T}\|_{2} + \|\sum_{i=0}^{n-1} \Big( \sigma(y.f_{\text{R}}^{i}(\mathbf{x})) - 1\Big) \epsilon~\text{sign}(\mathbf{w}_{\text{R}}^{i}) \|_{2} \tag{Using Triangle Inequality}\\
    &\leq \|\sum_{i=0}^{n-1} \Big(\sigma(y.f_{\text{NR}}^{i}(\mathbf{x}) - \sigma(y.f_{\text{R}}^{i}(\mathbf{x}))\Big)y.\mathbf{x}^{T}\|_{2} + \epsilon \sqrt{d} \sum_{i=0}^{n-1} \| \sigma(y.f_{\text{R}}^{i}(\mathbf{x})) - 1 \|_{2} \\
    &\leq \sum_{i=0}^{n-1}\| \Big(\sigma(y.f_{\text{NR}}^{i}(\mathbf{x}) - \sigma(y.f_{\text{R}}^{i}(\mathbf{x}))\Big) y.\mathbf{x}^{T}\|_{2} + \epsilon \sqrt{d} \sum_{i=0}^{n-1} \| \sigma(y.f_{\text{R}}^{i}(\mathbf{x})) - 1 \|_{2} \\
    &\leq \sum_{i=0}^{n-1} \|\Big(\sigma(y.f_{\text{NR}}^{i}(\mathbf{x})) - \sigma(y.f_{\text{R}}^{i}(\mathbf{x}))\Big)\|_{2} \|y.\mathbf{x}^{T}\|_{2} + \epsilon \sqrt{d} \sum_{i=0}^{n-1} \| \sigma(y.f_{\text{R}}^{i}(\mathbf{x})) - 1 \|_{2} \\
    &\leq n \|y.\mathbf{x}^{T}\|_{2} + \epsilon \sqrt{d} \sum_{i=0}^{n-1} \| 1 - \sigma(y.f_{\text{R}}^{i}(\mathbf{x})) \|_{2} \tag{since the maximum value of $\Big(\sigma(y.f_{\text{NR}}^{i}(\mathbf{x})) - \sigma(y.f_{\text{R}}^{i}(\mathbf{x}))\Big)$ is one}\\
    &\leq n \|y.\mathbf{x}^{T}\|_{2} + \epsilon \sqrt{d} \sum_{i=0}^{n-1} (1 - \sigma(y.f_{\text{R}}^{i}(\mathbf{x}))) \tag{since the term inside $\|\cdot\|_{2}$ is a scalar}\\
    &\leq n \|y.\mathbf{x}^{T}\|_{2} + \epsilon \sqrt{d} n\\
    \|\mathbf{w}_{\text{NR}}^{n} - \mathbf{w}_{\text{R}}^{n}\|_2 &\leq n\eta ( y\|\mathbf{x}^{T}\|_{2} + \epsilon \sqrt{d})\tag{since the maximum value of $(1 - \sigma(y.f_{\text{R}}^{i}(\mathbf{x})))$ is one}
\end{flalign*}

Using reverse triangle inequality, we can write: $ \|\mathbf{w}_{\text{NR}}^{n} - \mathbf{w}_{\text{R}}^{n}\|_2 \geq \Big| \|\mathbf{w}_{\text{NR}}^{n}\|_2 - \|\mathbf{w}_{\text{R}}^{n}\|_2 \Big|$, \ie, $\|\mathbf{w}_{\text{NR}}^{n} - \mathbf{w}_{\text{R}}^{n}\|_2 \geq \|\mathbf{w}_{\text{NR}}^{n}\|_2 - \|\mathbf{w}_{\text{R}}^{n}\|_2$ and $\|\mathbf{w}_{\text{NR}}^{n} - \mathbf{w}_{\text{R}}^{n}\|_2 \geq \|\mathbf{w}_{\text{R}}^{n}\|_2 - \|\mathbf{w}_{\text{NR}}^{n}\|_2$. We can now write the difference between the weights of the \vanilla and \robust models at the $n$-th iteration as:
\begin{flalign*}
    \Big| \|\mathbf{w}_{\text{NR}}^{n}\|_2 - \|\mathbf{w}_{\text{R}}^{n}\|_2 \Big| &\leq n\eta ( y\|\mathbf{x}^{T}\|_{2} + \epsilon \sqrt{d}) \tag{follows from reverse triangle inequality}\\
    -n\eta ( y\|\mathbf{x}^{T}\|_{2} + \epsilon \sqrt{d}) &\leq \|\mathbf{w}_{\text{NR}}^{n}\|_2 - \|\mathbf{w}_{\text{R}}^{n}\|_2  \leq n\eta ( y\|\mathbf{x}^{T}\|_{2} + \epsilon \sqrt{d})\\
    -n\eta ( y\|\mathbf{x}^{T}\|_{2} + \epsilon \sqrt{d}) &\leq \|\mathbf{w}_{\text{R}}^{n}\|_2 - \|\mathbf{w}_{\text{NR}}^{n}\|_2  \leq n\eta ( y\|\mathbf{x}^{T}\|_{2} + \epsilon \sqrt{d})\\
    \|\mathbf{w}_{\text{NR}}^{n}\|_2 - \Delta &\leq \|\mathbf{w}_{\text{R}}^{n}\|_2  \leq \|\mathbf{w}_{\text{NR}}^{n}\|_2 + \Delta,
\end{flalign*}
where $\Delta = n\eta ( y\|\mathbf{x}^{T}\|_{2} + \epsilon \sqrt{d})$.
\end{proof}

\subsection{Proof for Theorem~\ref{thm:cost-bound-linear-sketch}}
\label{sec:proof-thm1-full}

\begin{theorem}  (Cost difference of SCFE for linear models) For a given instance $\bx$, let $\bx'_{\textup{NR}}=\bx+\zeta_{\textup{NR}}$ and $\bx'_{\textup{R}}=\bx+\zeta_{\textup{R}}$ be the recourse generated using Wachter's algorithm for the \vanilla and \robust linear models. Then, for a normalized Lipschitz activation function $\sigma(\cdot)$, the cost difference in the recourse for both models can be bounded as:
\begin{equation}
    \|\zeta_{\textup{NR}}\|_2 - \|\zeta_{\textup{R}}\|_2
    \leq \Big|~\lambda\frac{2\norm{\bw_{\textup{NR}}}_{2} + \Delta}{\norm{\bw_{\textup{NR}}}_{2} (\|\mathbf{w}_{\textup{NR}}\|_2 - \Delta)}~\Big|
    \label{eq:scfe_bound}
\end{equation}
where $\bw_{\textup{NR}}$ is the weight of the \vanilla model, $\lambda$ is the scalar coefficient on the distance between original sample $\bx$ and generated counterfactual $\bx'$, and $\Delta$ is the weight difference between the weights of the \vanilla and \robust linear models from Lemma~\ref{thm:weight-linear-full}.
\label{thm:cost-bound-linear-full}
\end{theorem}

\begin{proof}
Following the definition of SCFE in Equation~\ref{eq:scfe}, we can find a counterfactual sample $\bx'$ that is "closest" to the original instance $\bx$ by minimizing the following objective:
\begin{equation}
    \argmin_{\mathbf{x'}} (f(\bx') - y')^2 + \lambda d(\bx', \bx),
    \label{eq:wachter}
\end{equation}
where $s$ is the target score, $\lambda$ is the regularization coefficient, and $d(\cdot)$ is the distance between the original and counterfactual sample $\mathbf{x'}$. 
 
Using the optimal recourse cost (Definition~\ref{thm:definition-optimal}), we can derive the upper bound of the cost difference for generating recourses using \vanilla and \robust models:
\begin{align*}
    \| \zeta_{\text{NR}} - \zeta_{\text{R}}\|_{2} &= \norm{ \frac{(s - \bw_{\text{NR}}^{T}\bx ) \lambda}{\lambda + \norm{\bw_{\text{NR}}}^{2}_{2}}\cdot \bw_{\text{NR}}  - \frac{(s - \bw_{\text{R}}^{T}\bx) \lambda}{\lambda + \norm{\bw_{\text{R}}}^{2}_{2}}\cdot \bw_{\text{R}}}_{2} \\
    &= \norm{ \frac{(s - \bw_{\text{NR}}^{T}\bx ) \lambda}{\lambda + \norm{\bw_{\text{NR}}}^{2}_{2}}\cdot \bw_{\text{NR}}  + \frac{(\bw_{\text{R}}^{T}\bx - s ) \lambda}{\lambda + \norm{\bw_{\text{R}}}^{2}_{2}}\cdot \bw_{\text{R}}}_{2} \\
    & \leq  \left | \frac{(s - \bw_{\text{NR}}^{T}\bx ) \lambda}{\lambda + \norm{\bw_{\text{NR}}}^{2}_{2}} \right |  \norm{\bw_{\text{NR}}}_{2}  + \left | \frac{(s - \bw_{\text{R}}^{T}\bx ) \lambda}{\lambda + \norm{\bw_{\text{R}}}^{2}_{2}} \right |  \norm{\bw_{\text{R}}}_{2} \tag{Using Triangle Inequality and norm properties}
\end{align*}
Note that the difference between the target and the predicted score for both \vanilla and \robust models is upper bounded by a term that is always positive. Hence, we get:
\begin{align*}
    \| \zeta_{\text{NR}} - \zeta_{\text{R}}\|_{2} 
    & \leq \left | \frac{\lambda}{\lambda + \norm{\bw_{\text{NR}}}^{2}_{2}} \right | \norm{\bw_{\text{NR}}}_{2} + \left | \frac{\lambda}{\lambda + \norm{\bw_{\text{R}}}^{2}_{2}} \right |  \norm{\bw_{\text{R}}}_{2}\\
    & \leq \frac{\lambda}{\norm{\bw_{\text{NR}}}_{2}} + \frac{\lambda}{\norm{\bw_{\text{R}}}_{2}} \tag{since $\lambda$ is a small positive constant}\\
    & \leq \lambda\frac{\norm{\bw_{\text{NR}}}_{2} + \norm{\bw_{\text{R}}}_{2}}{\norm{\bw_{\text{NR}}}_{2}\norm{\bw_{\text{R}}}_{2}}\\
    & \leq \lambda\frac{2\norm{\bw_{\text{NR}}}_{2} + \Delta}{\norm{\bw_{\text{NR}}}_{2} (\|\mathbf{w}_{\text{NR}}\|_2 - \Delta)} \tag{from Lemma~\ref{thm:weight-non-linear-full}}
\end{align*}

Using reverse triangle inequality, we can write: $ \|\zeta_{\text{NR}} - \zeta_{\text{R}}\|_2 \geq \Big| \|\zeta_{\text{NR}}\|_2 - \|\zeta_{\text{R}}\|_2 \Big|$, \ie, $\|\zeta_{\text{NR}} - \zeta_{\text{R}}\|_2 \geq \|\zeta_{\text{NR}}\|_2 - \|\zeta_{\text{R}}\|_2$ and $\|\zeta_{\text{NR}} - \zeta_{\text{R}}\|_2 \geq \|\zeta_{\text{R}}\|_2 - \|\zeta_{\text{NR}}\|_2$. We can now write the cost difference of recourses generated for the \vanilla and \robust models as:
\begin{flalign*}
    \Big| \|\zeta_{\text{NR}}\|_2 - \|\zeta_{\text{R}}\|_2 \Big| &\leq \lambda\frac{2\norm{\bw_{\text{NR}}}_{2} + \Delta}{\norm{\bw_{\text{NR}}}_{2} (\|\mathbf{w}_{\text{NR}}\|_2 - \Delta)} \tag{follows from reverse triangle inequality}\\
    -\lambda\frac{2\norm{\bw_{\text{NR}}}_{2} + \Delta}{\norm{\bw_{\text{NR}}}_{2} (\|\mathbf{w}_{\text{NR}}\|_2 - \Delta)} &\leq \|\zeta_{\text{NR}}\|_2 - \|\zeta_{\text{R}}\|_2  \leq \lambda\frac{2\norm{\bw_{\text{NR}}}_{2} + \Delta}{\norm{\bw_{\text{NR}}}_{2} (\|\mathbf{w}_{\text{NR}}\|_2 - \Delta)},
\end{flalign*}
where $\Delta = n\eta ( y\|\mathbf{x}^{T}\|_{2} + \epsilon \sqrt{d})$.
\end{proof}

\subsection{Proof for Theorem~\ref{thm:cost-bound-non-linear-sketch}}
\label{sec:proof-thm2-full}
\begin{theorem} (Cost difference for SCFE for wide neural network) For an NTK model with weights $\bw^{\textup{NTK}}_{\textup{NR}}$ and $\bw^{\textup{NTK}}_{\textup{R}}$ for the \vanilla and \robust models, the cost difference between the recourses generated for sample $\bx$ is bounded as:
    \begin{equation}
         \|\zeta_{\textup{NR}}\|_2 - \|\zeta_{\textup{R}}\|_2 \leq \Big|~\frac{2}{\textup{H}(\|\bar{\bw}^{\textup{NTK}}_{\textup{NR}}\|_{2},\|\bar{\bw}^{\textup{NTK}}_{\textup{R}}\|_{2})}~\Big|,
         \label{eq:proof-non-linear-cost-full}
    \end{equation}
    where $\textup{H}(\cdot, \cdot)$ denotes the harmonic mean, $\bar{\bw}^{\textup{NTK}}_{\textup{NR}}{=}\nabla_{\bx}\bK^{\infty}(\bx, \bX)\bw^{\textup{NTK}}_{\textup{NR}}$, $\bK^{\infty}$ is the NTK associated with the wide neural network model, $\bar{\bw}^{\textup{NTK}}_{\textup{R}}{=}\nabla_{\bx}\bK^{\infty}(\bx, \bX_{\textup{R}})\bw^{\textup{NTK}}_{\textup{R}}$, $\bw^{\textup{NTK}}_{\textup{NR}}{=}(\bK^{\infty}(\bX, \bX){+}\beta\bI_{n})^{-1}\bY$, $\bw^{\textup{NTK}}_{\textup{R}}{=}(\bK^{\infty}(\bX_{\textup{R}}, \bX_{\textup{R}}) + \beta\bI_{n})^{-1}\bY$, $\beta$ is the bias of the NTK model, $(\bX, \bX_{\textup{R}})$ are the training samples for the \vanilla and \robust models, and $\bY$ are the labels of the training samples.
    \label{thm:cost-bound-non-linear-full}
\end{theorem}

\begin{proof}
    For a wide neural network model with ReLU activations, the prediction for a given input $\bx$ is given by:
    \begin{equation}
        f^{\text{NTK}}(\bx) = \Big(\bK^{\infty}(\bx, \bX)\Big)^{\text{T}}\bw^{\text{NTK}},
    \end{equation}
    where $\bX$ is the training dataset and the NTK weights $\bw^{\text{NTK}}=(\bK^{\infty}(\bX, \bX) + \beta \bI_{n})^{-1}\bY$.
    To this end, the first order Taylor expansion of a wide neural network prediction model for an infinitesimal cost $\zeta$ is given as $f^{\text{NTK}}(\bx + \zeta)\approx f^{\text{NTK}}(\bx) + \zeta^{\text{T}}\bar{\bw}^{\textup{NTK}}$, where $\bar{\bw}^{\textup{NTK}}=\nabla_{\bx}\bK^{\infty}(\bx, \bX)\bw^{\textup{NTK}}$~\citep{pawelczyk2022trade} and the cost of generating a recourse of the corresponding NTK is given as $\zeta = \frac{s-f(\bx)}{\lambda + \|\bar{\bw}^{\text{NTK}}\|^{2}_{2}}\bar{\bw}^{\text{NTK}}$.
    
    Without any loss of generality, the cost for generating a valid recourse for a \vanilla ($\zeta_{\text{NR}}$) and \robust ($\zeta_{\text{R}}$) model can be written as:
    \begin{eqnarray}
        \zeta_{\text{NR}} = \frac{s-f^{\text{NTK}}_{\text{NR}}(\bx)}{\lambda + \|\bar{\bw}^{\text{NTK}}_{\text{NR}}\|^{2}_{2}}\bar{\bw}^{\text{NTK}}_{\text{NR}}\\
        \zeta_{\text{R}} = \frac{s-f^{\text{NTK}}_{\text{R}}(\bx)}{\lambda + \|\bar{\bw}^{\text{NTK}}_{\text{R}}\|^{2}_{2}}\bar{\bw}^{\text{NTK}}_{\text{R}}
    \end{eqnarray}

Hence, the cost difference between the recourses generated for \vanilla and \robust models is:
\begin{flalign}
    \zeta_{\text{NR}} - \zeta_{\text{R}} &= \frac{s-f^{\text{NTK}}_{\text{NR}}(\bx)}{\lambda + \|\bar{\bw}^{\text{NTK}}_{\text{NR}}\|^{2}_{2}}\bar{\bw}^{\text{NTK}}_{\text{NR}} - \frac{s-f^{\text{NTK}}_{\text{R}}(\bx)}{\lambda + \|\bar{\bw}^{\text{NTK}}_{\text{R}}\|^{2}_{2}}\bar{\bw}^{\text{NTK}}_{\text{R}}\label{app:cost-diff-non-linear}
\end{flalign}
Following Eqn.~\ref{app:cost-diff-non-linear} and taking $\ell_2$-norm on both sides, we can now compute the upper bound for the cost difference between the recourses generated for \vanilla and \robust models:
\begin{flalign}
    \|\zeta_{\text{NR}} - \zeta_{\text{R}}\|_2 &= \|~\frac{s-f^{\text{NTK}}_{\text{NR}}(\bx)}{\lambda + \|\bar{\bw}^{\text{NTK}}_{\text{NR}}\|^{2}_{2}}\bar{\bw}^{\text{NTK}}_{\text{NR}} + \frac{f^{\text{NTK}}_{\text{R}}(\bx)-s}{\lambda + \|\bar{\bw}^{\text{NTK}}_{\text{R}}\|^{2}_{2}}\bar{\bw}^{\text{NTK}}_{\text{R}}~\|_2\\
    \|\zeta_{\text{NR}} - \zeta_{\text{R}}\|_2 &\leq \|~\frac{s-f^{\text{NTK}}_{\text{NR}}(\bx)}{\lambda + \|\bar{\bw}^{\text{NTK}}_{\text{NR}}\|^{2}_{2}}\bar{\bw}^{\text{NTK}}_{\text{NR}}\|_2 + \|\frac{f^{\text{NTK}}_{\text{R}}(\bx)-s}{\lambda + \|\bar{\bw}^{\text{NTK}}_{\text{R}}\|^{2}_{2}}\bar{\bw}^{\text{NTK}}_{\text{R}}~\|_2 \tag{Using triangle inequality}\\
    \|\zeta_{\text{NR}} - \zeta_{\text{R}}\|_2 &\leq \frac{\|\bar{\bw}^{\text{NTK}}_{\text{NR}}\|_2}{\lambda + \|\bar{\bw}^{\text{NTK}}_{\text{NR}}\|^{2}_{2}} + \frac{\|\bar{\bw}^{\text{NTK}}_{\text{R}}~\|_2}{\lambda + \|\bar{\bw}^{\text{NTK}}_{\text{R}}\|^{2}_{2}} \tag{for reliable recourses and using $\|k\bx\|_2 = |k|\|\bx\|_2$}\\
    \|\zeta_{\text{NR}} - \zeta_{\text{R}}\|_2 &\leq \frac{1}{\|\bar{\bw}^{\text{NTK}}_{\text{NR}}\|_{2}} + \frac{1}{\|\bar{\bw}^{\text{NTK}}_{\text{R}}\|_{2}}\tag{since $\lambda$ is a small positive constant}
\end{flalign}

Again, using reverse triangle inequality, we can write the cost difference of recourses generated for the \vanilla and \robust models as:
\begin{flalign*}
    \Big| \|\zeta_{\text{NR}}\|_2 - \|\zeta_{\text{R}}\|_2 \Big| &\leq \frac{1}{\|\bar{\bw}^{\text{NTK}}_{\text{NR}}\|_{2}} + \frac{1}{\|\bar{\bw}^{\text{NTK}}_{\text{R}}\|_{2}}  \tag{follows from reverse triangle inequality}\\
    -\Big(\frac{1}{\|\bar{\bw}^{\text{NTK}}_{\text{NR}}\|_{2}} + \frac{1}{\|\bar{\bw}^{\text{NTK}}_{\text{R}}\|_{2}} \Big) &\leq \|\zeta_{\text{NR}}\|_2 - \|\zeta_{\text{R}}\|_2  \leq \frac{1}{\|\bar{\bw}^{\text{NTK}}_{\text{NR}}\|_{2}} + \frac{1}{\|\bar{\bw}^{\text{NTK}}_{\text{R}}\|_{2}}\\
    -\frac{2}{\textup{H}(\|\bar{\bw}^{\textup{NTK}}_{\textup{NR}}\|_{2},\|\bar{\bw}^{\textup{NTK}}_{\textup{R}}\|_{2})} &\leq \|\zeta_{\textup{NR}}\|_2 - \|\zeta_{\textup{R}}\|_2 \leq \frac{2}{\textup{H}(\|\bar{\bw}^{\textup{NTK}}_{\textup{NR}}\|_{2},\|\bar{\bw}^{\textup{NTK}}_{\textup{R}}\|_{2})},
\end{flalign*}
where $\textup{H}(\|\bar{\bw}^{\textup{NTK}}_{\textup{NR}}\|_{2},\|\bar{\bw}^{\textup{NTK}}_{\textup{R}}\|_{2})$ represents the harmonic mean of $\|\bar{\bw}^{\textup{NTK}}_{\textup{NR}}\|_{2}$ and $\|\bar{\bw}^{\textup{NTK}}_{\textup{R}}\|_{2}$.
\end{proof}

\subsection{Proof for Theorem~\ref{thm:cost-bound-cchvae-sketch}}
\label{sec:proof-thm3}

\begin{theorem} (Cost difference for C-CHVAE)
    Let $\bz_{\textup{NR}}'$ and $\bz_{\textup{R}}'$ be the solution returned by the \chvae recourse method by sampling from $\ell_{p}$-norm ball in the latent space using an L$_{G}$-Lipschitz decoder $\cG(\cdot)$ for a \vanilla and \robust model. By definition of the recourse method, let $\bx_{\textup{NR}}'{=}\cG(\bz_{\textup{NR}}'){=}\bx+\zeta_{\textup{NR}}$ and $\bx_{\textup{R}}'{=}\cG(\bz_{\textup{R}}'){=}\bx+\zeta_{\textup{R}}$ be the respective recourses in the input space whose difference can then be bounded as:
    \begin{equation}
         \|\zeta_{\textup{NR}}\|_2 - \|\zeta_{\textup{R}}\|_2 \leq \Big|L_{G}(r_{\textup{R}}+r_{\textup{NR}})\Big|,
         \label{eq:proof_cost_difference-full}
    \end{equation}
    where $r_{\textup{NR}}$ and $r_{\textup{R}}$ be the corresponding radii chosen by the algorithm such that they successfully return a recourse for the \vanilla and \robust model.
\end{theorem}

\begin{proof}
From the formulation of the counterfactual algorithm, we can write the difference between the counterfactuals $\bx_{\textup{R}}'$ and $\bx_{\textup{NR}}'$ as:
\begin{align}
        \|\bx_{\text{R}}' - \bx_{\text{NR}}'\|_{p} &= \|\cG(\bz_{\text{R}}') - \cG(\bz_{\text{NR}}')\|_{p}
   \label{eq:latent}
\end{align}

Using Equation~\ref{eq:latent}, we can derive the upper bound of the cost difference using triangle inequality properties.
\begin{align}
    \|\bx_{\text{R}}' - \bx_{\text{NR}}'\|_{p} &= \|\cG(\bz_{\text{R}}') - \bx + \bx - \cG(\bz_{\text{NR}}')\|_p \\
    \|\bx_{\text{R}}' - \bx_{\text{NR}}'\|_{p} &\leq \|\cG(\bz_{\text{R}}') - \bx\|_p + \|\bx - \cG(\bz_{\text{NR}}')\|_p \tag{using triangle inequality}\\
    &= \|\cG(\bz_{\text{R}}') - \cG(\bz)\|_p + \|\cG(\bz) - \cG(\bz_{\text{NR}}')\|_p \\
    &\leq L_{G}\|\bz_{\text{R}}'- \bz \|_p + L_{G}\|\bz - \bz_{\text{NR}}'\|_p \tag{from Bora et al.~\citep{bora2017compressed}} \\ 
    \|\bx_{\text{R}}' - \bx_{\text{NR}}'\|_{p} &\leq L_{\text{G}}(r_{\text{R}}+r_{\text{NR}}),
    \label{eq:cchvae-ub}
\end{align}
where $r_{\text{R}}$ and $r_{\text{NR}}$ are the radius of the $\ell_{p}$-norm for generating samples from the robust and baseline model, respectively.

Following Theorems 1-2, we use reverse triangle inequality on Equation~\ref{eq:cchvae-ub} and get the lower and upper bounds for the cost difference of C-CHVAE recourses generated using non-robust and adversarially robust models as:
\begin{flalign*}
    \Big| \|\zeta_{\text{NR}}\|_2 - \|\zeta_{\text{R}}\|_2 \Big| &\leq L_{\text{G}}(r_{\text{R}}+r_{\text{NR}})\\
    -L_{\text{G}}(r_{\text{R}}+r_{\text{NR}}) &\leq \|\zeta_{\text{NR}}\|_2 - \|\zeta_{\text{R}}\|_2  \leq L_{\text{G}}(r_{\text{R}}+r_{\text{NR}})
\end{flalign*}
\end{proof}

\subsection{Proof for Theorem~\ref{thm:valid-linear-sketch}}
\label{app:validity-linear}

\begin{theorem} (Validity comparison for linear model) For a given instance $\mathbf{x} \in \mathbb{R}^{d}$ and desired target label denoted by unity, let $\bx_{\textup{R}}'$ and $\bx_{\textup{NR}}'$ be the counterfactuals for \robust $f_{\textup{R}}(\bx)$ and \vanilla $f_{\textup{NR}}(\bx)$ models, respectively. Then, $\Pr(f_{\textup{NR}}(\bx_{\textup{NR}}') = 1) \geq \Pr(f_{\textup{R}}(\bx_{\textup{R}}') = 1)$ if $|f_{\textup{NR}}(\bx_{\textup{R}}') - f_{\textup{NR}}(\bx_{\textup{NR}}')| \leq \Delta \norm{\bx_{\textup{R}}'}_{2}$.
\label{thm:valid-linear-full}
\end{theorem}
\begin{proof}
For a linear model case, $\Pr(f(\bx) = 1) = \frac{e^{\bw^{T}\bx }}{1+ e^{\bw^{T}\bx}}$, where we take the sigmoid of the model output $\bw^{T}\bx$. Next, we derive the difference in probability of a valid recourse from \vanilla and \robust model:
\begin{align}
    \Pr(f_{\textup{NR}}(\bx_{\textup{NR}}') = 1) - \Pr(f_{\textup{R}}(\bx_{\textup{R}}') = 1) & = \frac{e^{\bw_{\textup{NR}}^{T}\bx_{\textup{NR}}'}}{1+ e^{\bw_{\textup{NR}}^{T}\bx_{\textup{NR}}'}} - \frac{e^{\bw_{\textup{R}}^{T}\bx_{\textup{R}}'}}{1+ e^{\bw_{\textup{R}}^{T}\bx_{\textup{R}}'}}
    \\ & = \frac{e^{\bw_{\textup{NR}}^{T}\bx_{\textup{NR}}'} - e^{\bw_{\textup{R}}^{T}\bx_{\textup{R}}'}}{(1+ e^{\bw_{\textup{R}}^{T}\bx_{\textup{R}}'})(1+ e^{\bw_{\textup{NR}}^{T}\bx_{\textup{NR}}'})}
\end{align}
Since $(1+ e^{\bw_{\textup{R}}^{T}\bx_{\textup{R}}'})(1+ e^{\bw_{\textup{NR}}^{T}\bx_{\textup{NR}}'}) > 0$, so  $\Pr(f_{\textup{NR}}(\bx_{\textup{NR}}') = 1) \geq \Pr(f_{\textup{R}}(\bx_{\textup{R}}') = 1))$ occurs when,
\begin{flalign}
e^{\bw_{\textup{NR}}^{T}\bx_{\textup{NR}}'} &\geq e^{\bw_{\textup{R}}^{T}\bx_{\textup{R}}'} \\ 
\bw_{\textup{NR}}^{T}(\bx_{\textup{NR}}' - \bx_{\textup{R}}') &\geq (\bw_{\textup{R}}^{T} - \bw_{\textup{NR}}^{T}) \bx_{\textup{R}}' \tag{Taking natural logarithm on both sides} \\
\bw_{\textup{NR}}^{T}(\bx_{\textup{R}}' - \bx_{\textup{NR}}' ) &\leq ( \bw_{\textup{NR}}^{T} - \bw_{\textup{R}}^{T} ) \bx_{\textup{R}}' \\
\norm{\bw_{\textup{NR}}^{T}(\bx_{\textup{R}}' - \bx_{\textup{NR}}')}_{2} &\leq \norm{( \bw_{\textup{NR}}^{T} - \bw_{\textup{R}}^{T}) \bx_{\textup{R}}'}_{2} \tag{Taking $\ell_2$-norm on both sides}\\
\norm{\bw_{\textup{NR}}^{T}(\bx_{\textup{R}}' - \bx_{\textup{NR}}')}_{2} &\leq \norm{\bw_{\textup{NR}} - \bw_{\textup{R}}}_{2} \norm{\bx_{\textup{R}}'}_{2} \tag{Using Cauchy-Schwartz} \\
\norm{\bw_{\textup{NR}}^{T}(\bx_{\textup{R}}' - \bx_{\textup{NR}}')}_{2} &\leq \Delta\norm{\bx_{\textup{R}}'}_{2} \tag{From Lemma~\ref{thm:weight-linear-full} }\\
|f_{\textup{NR}}(\bx_{\textup{R}}') - f_{\textup{NR}}(\bx_{\textup{NR}}')| &\leq \Delta\norm{\bx_{\textup{R}}'}_{2}
\end{flalign}
\end{proof}

\subsection{Proof for Lemma~\ref{thm:weight-non-linear-sketch}}
\label{sec:proof-lemma2-full}

\begin{lemma}(Difference between \vanilla and \robust model weights for wide neural network) For a given NTK model, let $\bw_{\textup{NR}}^{\textup{NTK}}$ and $\bw_{\textup{R}}^{\textup{NTK}}$ be weights of the \vanilla and \robust model. Then, for a wide neural network model with ReLU activations, the difference in the weights can be bounded as:
\begin{equation}
    \|\bw^{\textup{NTK}}_{\textup{NR}} - \bw^{\textup{NTK}}_{\textup{R}}\|_2 \leq \|(\bK^{\infty}(\bX, \bX) + \beta\bI_{n})^{-1} - (\bK^{\infty}(\bX_{\textup{R}}, \bX_{\textup{R}}) + \beta\bI_{n})^{-1}\|_2 \|\bY\|_2
\end{equation}
where $\bK^{\infty}(\cdot~;~\cdot)$ is the kernel matrix for ReLU networks as defined in Definition~\ref{def:kernel-matrix}$ and (\bX, \bX_{\textup{R}})$ are the training samples for the \vanilla and \robust models.
\label{thm:weight-non-linear-full}
\end{lemma}

\begin{proof}
The NTK weights for a wide neural network model with ReLU activations are given as:
\begin{align*}
    \bw^{\textup{NTK}}_{\textup{NR}}=(\bK^{\infty}(\bX, \bX) + \beta\bI_{n})^{-1}\bY\\   
    \bw^{\textup{NTK}}_{\textup{R}}=(\bK^{\infty}(\bX_{\textup{R}}, \bX_{\textup{R}}) + \beta\bI_{n})^{-1}\bY,
\end{align*}
Using the above equations, we can now write the difference between the weights of the \vanilla and \robust NTK models as:
\begin{flalign}
    \bw^{\textup{NTK}}_{\textup{NR}} - \bw^{\textup{NTK}}_{\textup{R}} =(\bK^{\infty}(\bX, \bX) + \beta\bI_{n})^{-1}\bY - (\bK^{\infty}(\bX_{\textup{R}}, \bX_{\textup{R}}) + \beta\bI_{n})^{-1}\bY,
\end{flalign}
Taking $\ell_2$-norm on both sides, we get:
\begin{flalign}
    \|\bw^{\textup{NTK}}_{\textup{NR}} - \bw^{\textup{NTK}}_{\textup{R}}\|_2 &=\|((\bK^{\infty}(\bX, \bX) + \beta\bI_{n})^{-1} - (\bK^{\infty}(\bX_{\textup{R}}, \bX_{\textup{R}}) + \beta\bI_{n})^{-1})\bY\|_2\\
    &\leq \|(\bK^{\infty}(\bX, \bX) + \beta\bI_{n})^{-1} - (\bK^{\infty}(\bX_{\textup{R}}, \bX_{\textup{R}}) + \beta\bI_{n})^{-1}\|_2 \|\bY\|_2 \tag{Using Cauchy-Schwartz}
 \end{flalign}

 Similar to Lemma~\ref{thm:weight-linear-full}, we can bound the difference between the weights of the \vanilla and \robust NTK models as:
 
 Using reverse triangle inequality, we can write: $ \|\bw^{\textup{NTK}}_{\textup{NR}} - \bw^{\textup{NTK}}_{\textup{R}}\|_2 \geq \Big| \|\bw^{\textup{NTK}}_{\textup{NR}}\|_2 - \|\bw^{\textup{NTK}}_{\textup{R}}\|_2 \Big|$, \ie, $\|\bw^{\textup{NTK}}_{\textup{NR}} - \bw^{\textup{NTK}}_{\textup{R}}\|_2 \geq \|\bw^{\textup{NTK}}_{\textup{NR}}\|_2 - \|\bw^{\textup{NTK}}_{\textup{R}}\|_2$ and $\|\bw^{\textup{NTK}}_{\textup{NR}} - \bw^{\textup{NTK}}_{\textup{R}}\|_2 \geq \|\bw^{\textup{NTK}}_{\textup{R}}\|_2 - \|\bw^{\textup{NTK}}_{\textup{NR}}\|_2$. Hence, the difference between the weights of the \vanilla and \robust NTK models as:
\begin{flalign*}
    \Big| \|\bw^{\textup{NTK}}_{\textup{NR}}\|_2 - \|\bw^{\textup{NTK}}_{\textup{R}}\|_2 \Big| &\leq \|(\bK^{\infty}(\bX, \bX) + \beta\bI_{n})^{-1} - (\bK^{\infty}(\bX_{\textup{R}}, \bX_{\textup{R}}) + \beta\bI_{n})^{-1}\|_2 \|\bY\|_2 \tag{follows from reverse triangle inequality}\\
    -\Delta_{\text{K}}\|\bY\|_2 &\leq \|\bw^{\textup{NTK}}_{\textup{NR}}\|_2 - \|\bw^{\textup{NTK}}_{\textup{R}}\|_2  \leq \Delta_{\text{K}}\|\bY\|_2 \tag{where $\Delta_{\text{K}} {=} \|(\bK^{\infty}(\bX, \bX) {+} \beta\bI_{n})^{-1} {-} (\bK^{\infty}(\bX_{\textup{R}}, \bX_{\textup{R}}) {+}\beta\bI_{n})^{-1}\|_2$}\\
    -\Delta_{\text{K}}\|\bY\|_2 &\leq \|\bw^{\textup{NTK}}_{\textup{R}}\|_2 - \|\bw^{\textup{NTK}}_{\textup{NR}}\|_2  \leq \Delta_{\text{K}}\|\bY\|_2\\
    \|\bw^{\textup{NTK}}_{\textup{NR}}\|_2 - \Delta_{\text{K}}\|\bY\|_2 &\leq \|\bw^{\textup{NTK}}_{\textup{R}}\|_2  \leq \|\bw^{\textup{NTK}}_{\textup{NR}}\|_2 + \Delta_{\text{K}}\|\bY\|_2
\end{flalign*}
\end{proof}

\subsection{Proof for Theorem~\ref{thm:valid-nonlinear-sketch}}
\label{app:validity-nonlinear}
\begin{theorem} (Validity Comparison for wide neural network) For a given instance $\bx \in \mathbb{R}^{d}$ and desired target label denoted by unity, let $\bx_{\textup{R}}'$ and $\bx_{\textup{NR}}'$ be the counterfactuals for \robust $f_{\textup{R}}^{\textup{NTK}}(\bx)$ and \vanilla $f_{\textup{NR}}^{\textup{NTK}}(\bx)$ wide neural network models respectively. Then, $\Pr(f_{\textup{NR}}^{\textup{NTK}}(\bx_{\textup{NR}}') = 1) \geq \Pr(f_{\textup{R}}^{\textup{NTK}}(\bx_{\textup{R}}') = 1)$ if $\norm{(\bK^{\infty}(\bx_{\textup{R}}', \bX_{\textup{R}}) - \bK^{\infty}(\bx_{\textup{NR}}', \bX) )^{\textup{T}} \bw_{\textup{NR}}^{\textup{NTK}}} \leq  \norm{\bK^{\infty}(\bx_{\textup{R}}', \bX_{\textup{R}})^{\textup{T}}} \Delta_{\textup{K}} \|\bY\|_2 $.
\label{thm:valid-nonlinear-full}
\end{theorem}

\begin{proof}
From the NTK model definition, we can write:
\begin{equation}
        f^{\text{NTK}}(\bx) = \Big(\bK^{\infty}(\bx, \bX)\Big)^{\text{T}}\bw^{\text{NTK}},
\end{equation}
For a wide neural network model with ReLU activations, $\Pr(f^{\text{NTK}}(\bx) = 1) = \frac{e^{\bK^{\infty}(\bx, \bX)^{\text{T}}\bw^{\text{NTK}} }}{1+ e^{\bK^{\infty}(\bx, \bX)^{\text{T}}\bw^{\text{NTK}}}}$, which is the sigmoid of the model output. Next, we derive the difference in probability of a valid recourse from \vanilla and \robust model:
\begin{align}
    \Pr(f_{\textup{NR}}^{\text{NTK}}(\bx_{\textup{NR}}') = 1) - \Pr(f_{\textup{R}}^{\text{NTK}}(\bx_{\textup{R}}') = 1) & = \frac{e^{\bK^{\infty}(\bx_{\textup{NR}}', \bX)^{\text{T}}\bw^{\text{NTK}}_{\text{NR}}}}{1+ e^{\bK^{\infty}(\bx_{\textup{NR}}', \bX)^{\text{T}}\bw^{\text{NTK}}_{\text{NR}}}} - \frac{e^{\bK^{\infty}(\bx_{\textup{R}}', \bX_{\text{R}})^{\text{T}}\bw^{\text{NTK}}_{\text{R}}}}{1+ e^{\bK^{\infty}(\bx_{\textup{R}}', \bX_{\text{R}})^{\text{T}}\bw^{\text{NTK}}_{\text{R}}}}
    \\ & = \frac{e^{\bK^{\infty}(\bx_{\textup{NR}}', \bX)^{\text{T}}\bw^{\text{NTK}}_{\text{NR}}} - e^{\bK^{\infty}(\bx_{\textup{R}}', \bX_{\text{R}})^{\text{T}}\bw^{\text{NTK}}_{\text{R}}}}{(1+ e^{\bK^{\infty}(\bx_{\textup{NR}}', \bX)^{\text{T}}\bw^{\text{NTK}}_{\text{NR}}})(1+ e^{\bK^{\infty}(\bx_{\textup{R}}', \bX_{\text{R}})^{\text{T}}\bw^{\text{NTK}}_{\text{R}}})}
\end{align}

Since $(1+ e^{\bK^{\infty}(\bx_{\textup{NR}}', \bX)^{\text{T}}\bw^{\text{NTK}}_{\text{NR}}})(1+ e^{\bK^{\infty}(\bx_{\textup{R}}', \bX_{\text{R}})^{\text{T}}\bw^{\text{NTK}}_{\text{R}}}) > 0$, so  $\Pr(f_{\textup{NR}}^{\text{NTK}}(\bx_{\textup{NR}}') = 1) \geq \Pr(f_{\textup{R}}^{\text{NTK}}(\bx_{\textup{R}}') = 1)$ occurs when,
\begin{flalign}
e^{\bK^{\infty}(\bx_{\textup{NR}}', \bX)^{\text{T}}\bw^{\text{NTK}}_{\text{NR}}} &\geq e^{\bK^{\infty}(\bx_{\textup{R}}', \bX_{\text{R}})^{\text{T}}\bw^{\text{NTK}}_{\text{R}}} \\ 
\bK^{\infty}(\bx_{\textup{NR}}', \bX)^{\text{T}}\bw^{\text{NTK}}_{\text{NR}} &\geq \bK^{\infty}(\bx_{\textup{R}}', \bX_{\text{R}})^{\text{T}}\bw^{\text{NTK}}_{\text{R}} \tag{Taking natural logarithm on both sides} \\ 
(\bK^{\infty}(\bx_{\textup{NR}}', \bX) - \bK^{\infty}(\bx_{\textup{R}}', \bX_{\text{R}}) )^{\text{T}} \bw_{\textup{NR}}^{\text{NTK}} &\geq \bK^{\infty}(\bx_{\textup{R}}', \bX_{\text{R}})^{\text{T}}(\bw_{\textup{R}}^{\text{NTK}} - \bw_{\textup{NR}}^{\text{NTK}})   \\
(\bK^{\infty}(\bx_{\textup{R}}', \bX_{\text{R}}) - \bK^{\infty}(\bx_{\textup{NR}}', \bX) )^{\text{T}} \bw_{\textup{NR}}^{\text{NTK}} &\leq \bK^{\infty}(\bx_{\textup{R}}', \bX_{\text{R}})^{\text{T}}(\bw_{\textup{NR}}^{\text{NTK}} - \bw_{\textup{R}}^{\text{NTK}})   \\
\norm{(\bK^{\infty}(\bx_{\textup{R}}', \bX_{\text{R}}) - \bK^{\infty}(\bx_{\textup{NR}}', \bX) )^{\text{T}} \bw_{\textup{NR}}^{\text{NTK}}} &\leq \norm{\bK^{\infty}(\bx_{\textup{R}}', \bX_{\text{R}})^{\text{T}}(\bw_{\textup{NR}}^{\text{NTK}} - \bw_{\textup{R}}^{\text{NTK}})} \tag{Taking norm on both sides}\\
\norm{(\bK^{\infty}(\bx_{\textup{R}}', \bX_{\text{R}}) - \bK^{\infty}(\bx_{\textup{NR}}', \bX) )^{\text{T}} \bw_{\textup{NR}}^{\text{NTK}}} &\leq \norm{\bK^{\infty}(\bx_{\textup{R}}', \bX_{\text{R}})^{\text{T}}} \norm{\bw_{\textup{NR}}^{\text{NTK}} - \bw_{\textup{R}}^{\text{NTK}}} \tag{Using Cauchy-Schwartz} \\
\norm{(\bK^{\infty}(\bx_{\textup{R}}', \bX_{\text{R}}) - \bK^{\infty}(\bx_{\textup{NR}}', \bX) )^{\text{T}} \bw_{\textup{NR}}^{\text{NTK}}} &\leq  \norm{\bK^{\infty}(\bx_{\textup{R}}', \bX_{\text{R}})^{\text{T}}} \Delta_{\text{K}} \|\bY\|_2  \tag{From Lemma~\ref{thm:weight-non-linear-full}}
\end{flalign}
\end{proof}

\section{Additional Experimental Results}
\label{app:results}

In this section, we have plots for cost differences, validity, and adversarial accuracy for the two logistic regression and neural network models trained on three real-world datasets. 

\begin{figure}[ht]
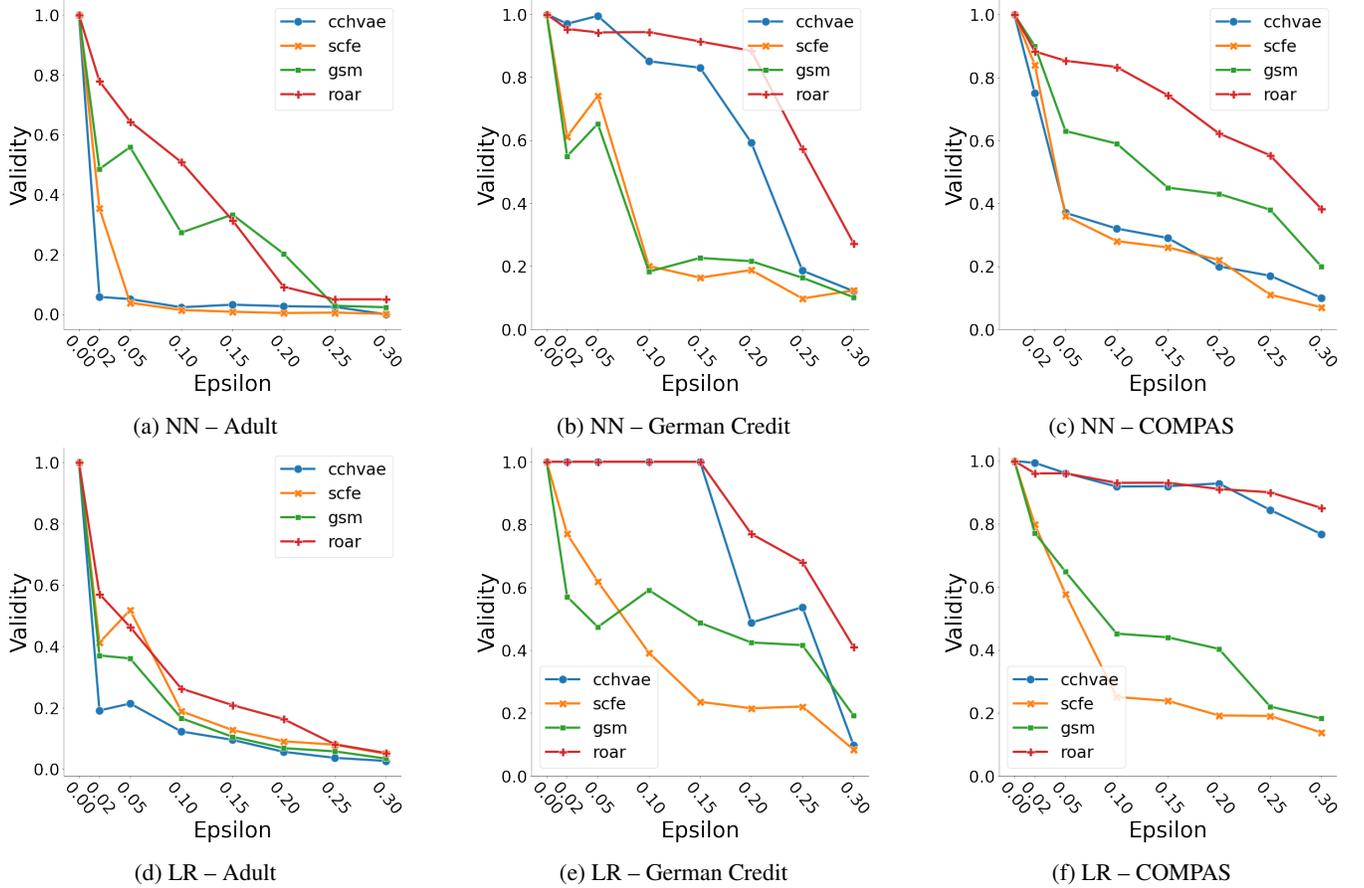

\centering
\begin{subfigure}{.3\linewidth}
    \centering
    \includegraphics[width=\textwidth]{results/adult/adult_nn_new_validity_final.png}
    \caption{NN -- Adult}\label{fig:validity1}
\end{subfigure}
    \hfill
\begin{subfigure}{.3\linewidth}
    \centering
    \includegraphics[width=\textwidth]{results/german/german_validity_nn.png}
    \caption{NN -- German Credit}\label{fig:validity2}
\end{subfigure}
   \hfill
\begin{subfigure}{.3\linewidth}
    \centering
    \includegraphics[width=\textwidth]{results/compas/compas_nn_validity.png}
    \caption{NN -- COMPAS}\label{fig:validity3}
\end{subfigure}
\bigskip
\begin{subfigure}{.3\linewidth}
    \centering
            \includegraphics[width=\textwidth]{results/adult/adult_lr_validity.png}
    \caption{LR -- Adult}\label{fig:validity4}
\end{subfigure}
    \hfill
\begin{subfigure}{.3\linewidth}
    \centering
    \includegraphics[width=\textwidth]{results/german/german_lr_validity.png}
    \caption{LR -- German Credit}\label{fig:validity5}
\end{subfigure}
   \hfill
\begin{subfigure}{.3\linewidth}
    \centering
    \includegraphics[width=\textwidth]{results/compas/compas_lr_validity.png}
    \caption{LR -- COMPAS}\label{fig:validity6}
\end{subfigure}
\caption{Analyzing validity of recourse generated using \vanilla and \robust Logistic Regression (LR) and Neural Networks (NN) for Adult, COMPAS, and German Credit datasets. We find that the validity decreases for increasing values of $\epsilon$.}\label{fig:validity-all}
\end{figure}

\begin{figure}[ht]
\centering
\begin{subfigure}{.3\linewidth}
    \centering
            \includegraphics[width=\textwidth]{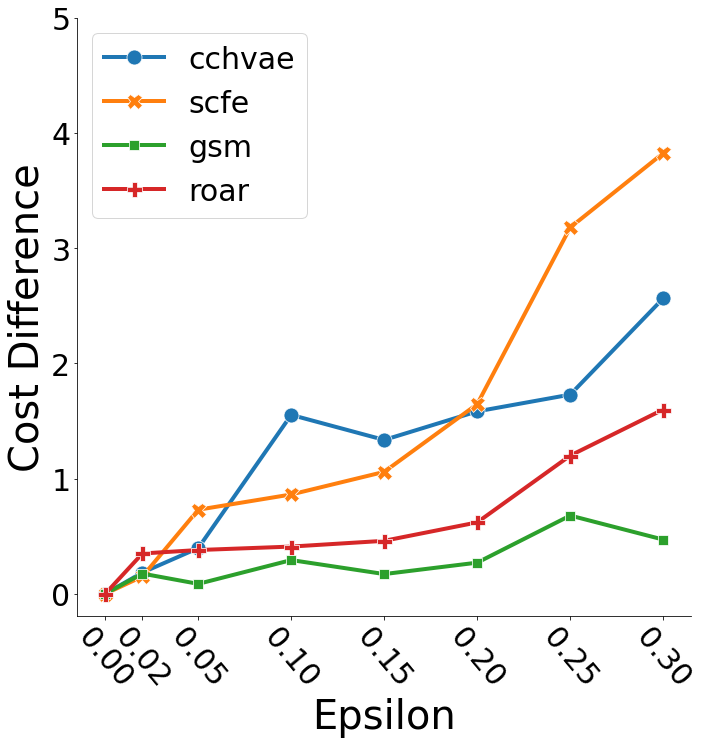}
    \caption{NN -- Adult}\label{fig:cost1}
\end{subfigure}
    \hfill
\begin{subfigure}{.3\linewidth}
    \centering
    \includegraphics[width=\textwidth]{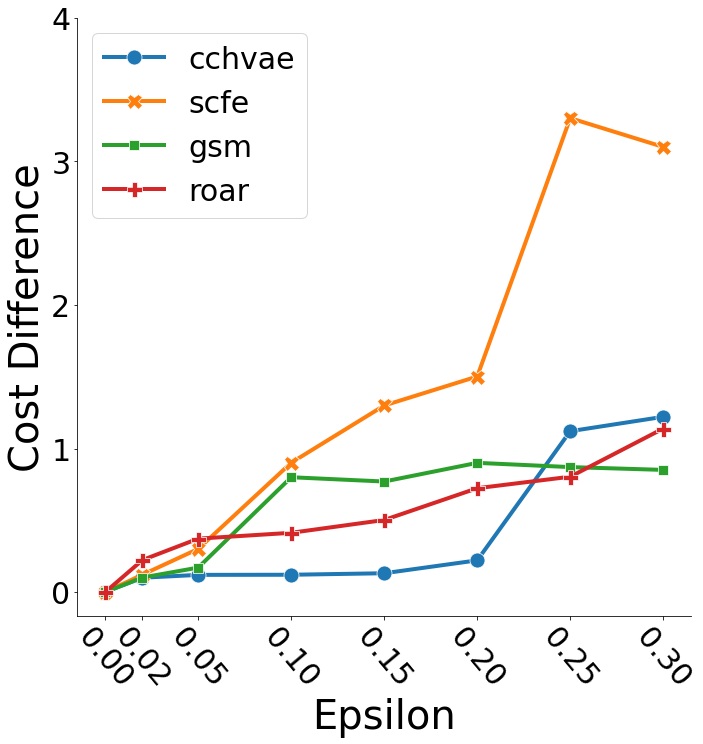}
    \caption{NN -- German Credit}\label{fig:cost2}
\end{subfigure}
   \hfill
\begin{subfigure}{.3\linewidth}
    \centering
    \includegraphics[width=\textwidth]{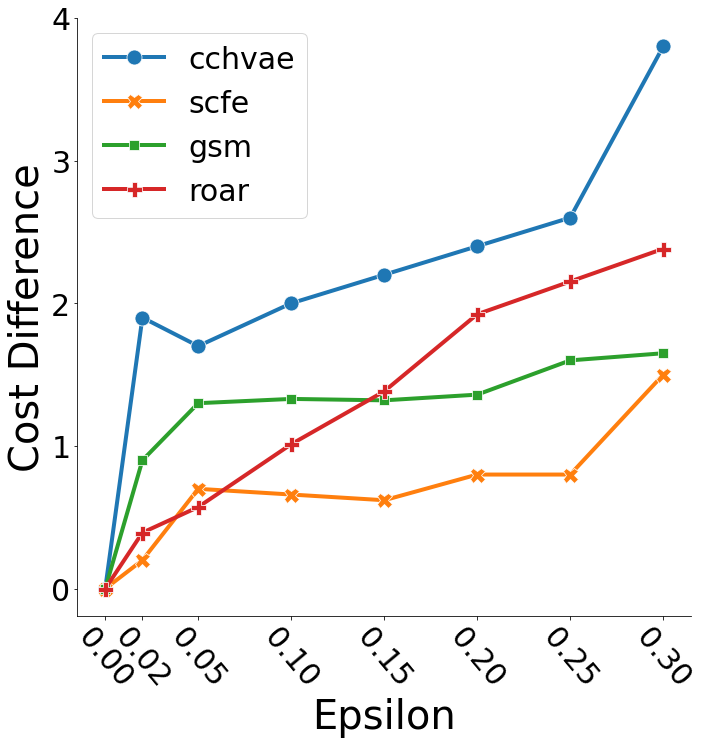}
    \caption{NN -- COMPAS}\label{fig:cost3}
\end{subfigure}

\bigskip
\begin{subfigure}{.3\linewidth}
    \centering
            \includegraphics[width=\textwidth]{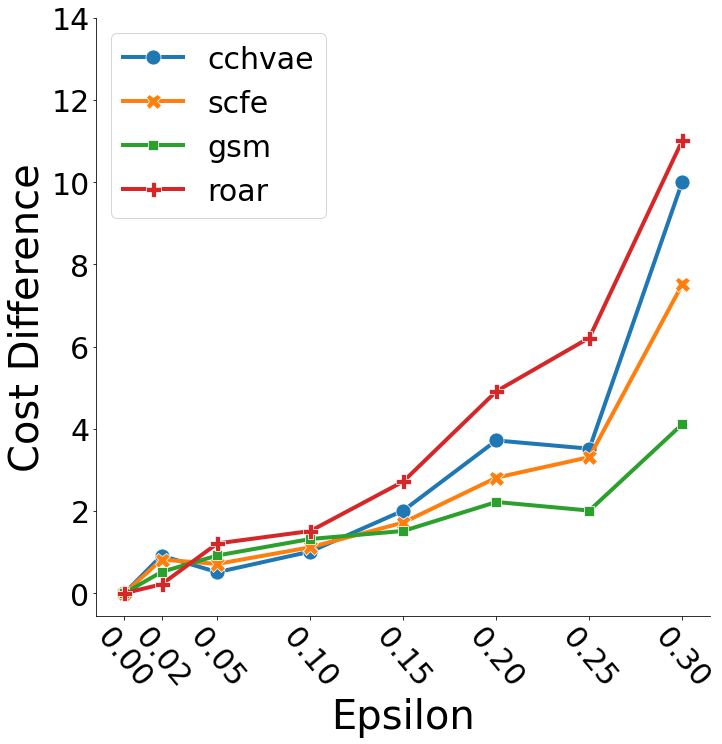}
    \caption{LR -- Adult}\label{fig:cost4}
\end{subfigure}
    \hfill
\begin{subfigure}{.3\linewidth}
    \centering
    \includegraphics[width=\textwidth]{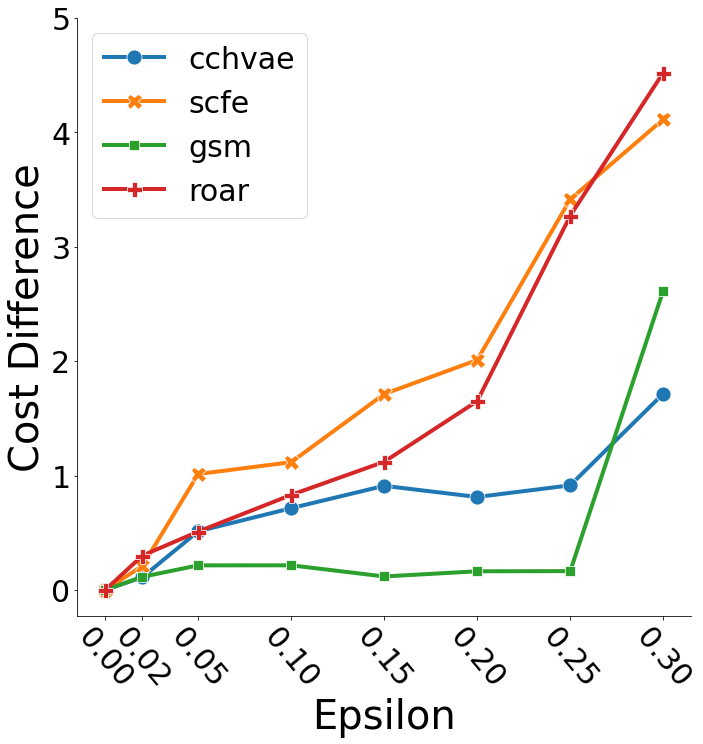}
    \caption{LR -- German Credit}\label{fig:cost5}
\end{subfigure}
   \hfill
\begin{subfigure}{.3\linewidth}
    \centering
    \includegraphics[width=\textwidth]{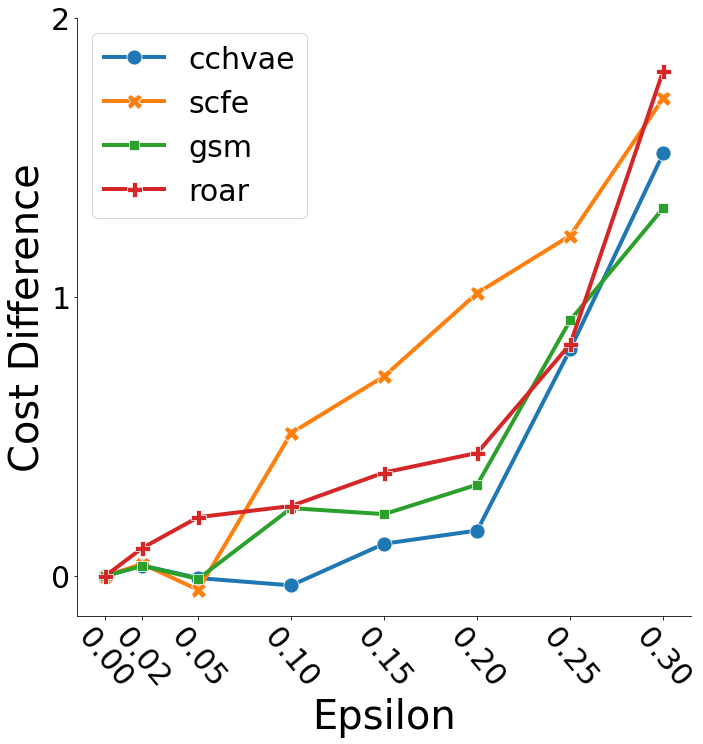}
    \caption{LR -- COMPAS}\label{fig:cost6}
\end{subfigure}
\caption{Analyzing cost differences between recourse generated using \vanilla and
\robust Logistic Regression (LR) and Neural Networks(NN) for Adult, COMPAS, and German Credit datasets. We find that the cost difference (i.e., $\ell_{2}-$norm) between the recourses generated for \vanilla and \robust models increases for increasing values of $\epsilon$.}
\label{fig:cost-analysis-all}
\end{figure}

\begin{figure}
\centering
\begin{subfigure}{.3\linewidth}
    \centering
            \includegraphics[width=\textwidth]{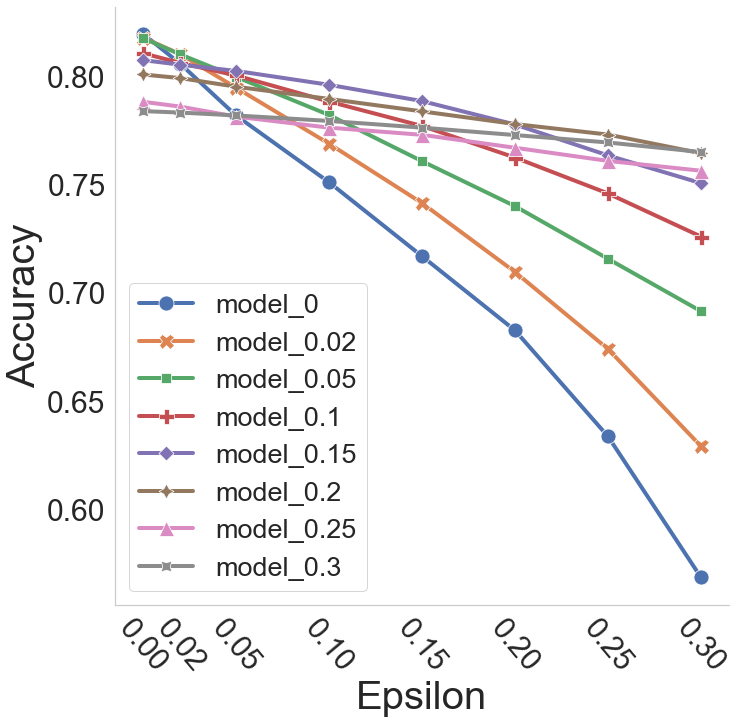}
    \caption{NN -- Adult}\label{fig:acc-adv1}
\end{subfigure}
    \hfill
\begin{subfigure}{.3\linewidth}
    \centering
    \includegraphics[width=\textwidth]{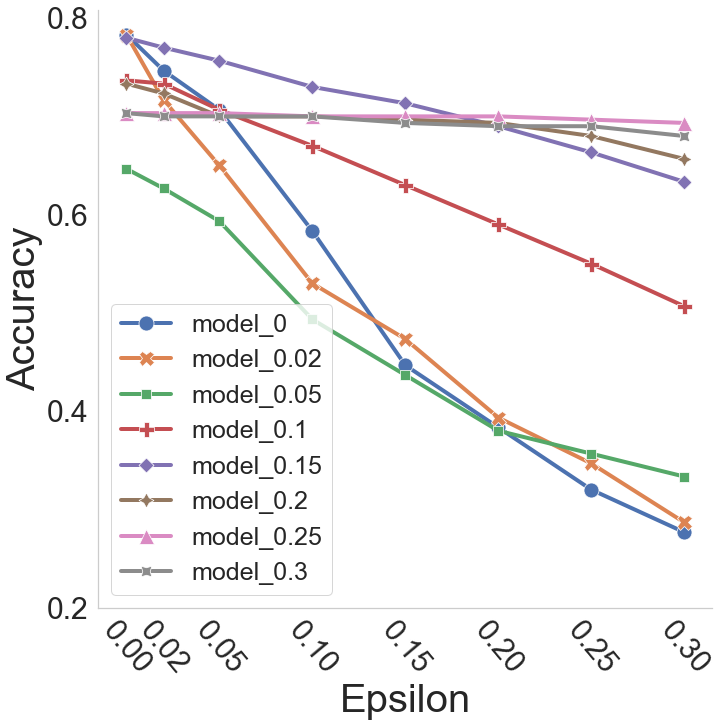}
    \caption{NN -- German Credit}\label{fig:acc-adv2}
\end{subfigure}
   \hfill
\begin{subfigure}{.3\linewidth}
    \centering
    \includegraphics[width=\textwidth]{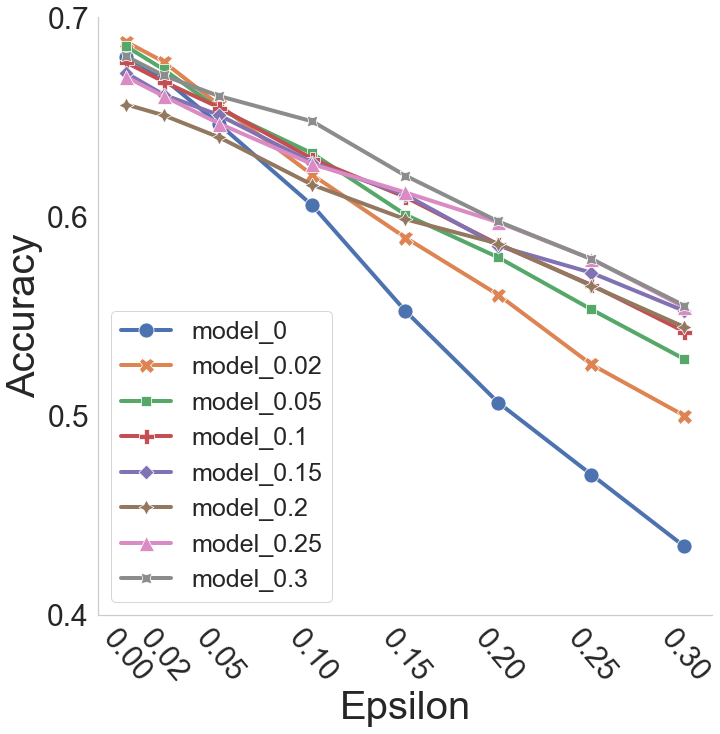}
    \caption{NN -- COMPAS}\label{fig:acc-adv3}
\end{subfigure}

\bigskip
\begin{subfigure}{.3\linewidth}
    \centering
            \includegraphics[width=\textwidth]{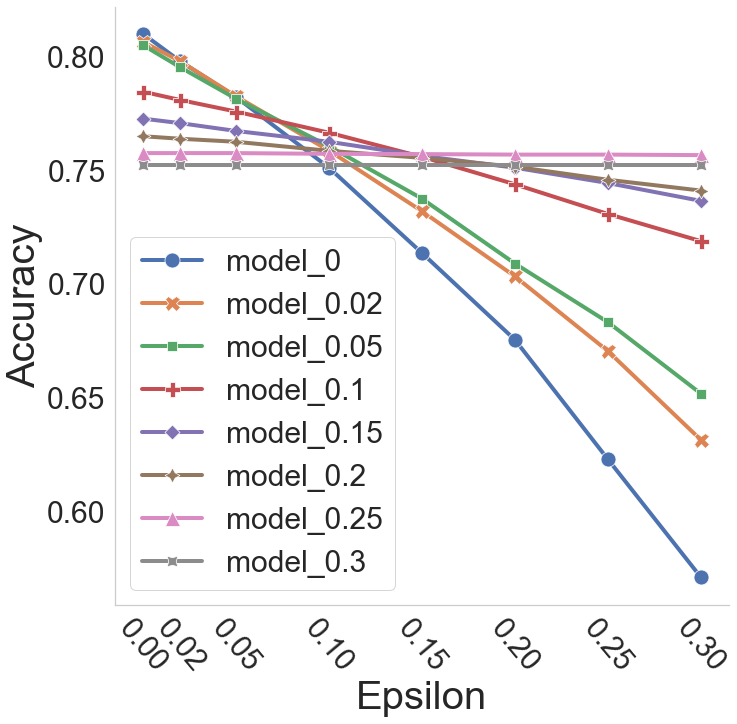}
    \caption{LR -- Adult}\label{fig:acc-adv4}
\end{subfigure}
    \hfill
\begin{subfigure}{.3\linewidth}
    \centering
    \includegraphics[width=\textwidth]{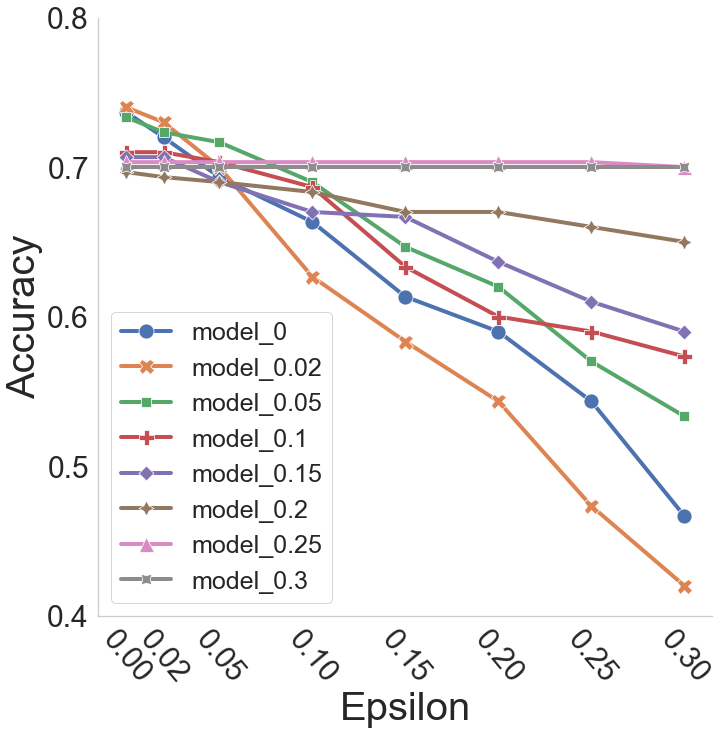}
    \caption{LR -- German Credit}\label{fig:acc-adv5}
\end{subfigure}
   \hfill
\begin{subfigure}{.3\linewidth}
    \centering
    \includegraphics[width=\textwidth]{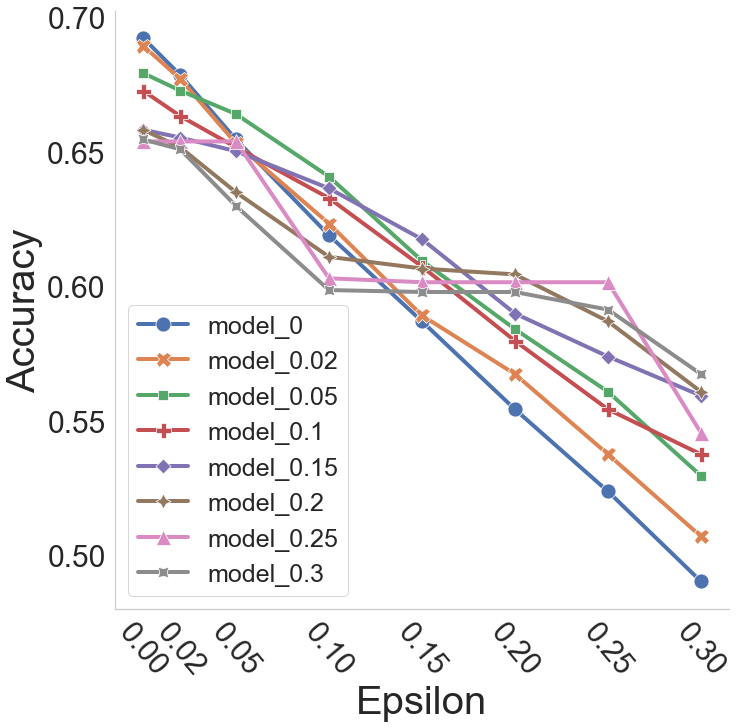}
    \caption{LR -- COMPAS}\label{fig:acc-adv6}
\end{subfigure}
\caption{ Here we plot the adversarial accuracy of the different models we trained on varying degree of robustness ($\epsilon$). As expected, we observe the adversarial accuracy for the \vanilla model is lowest out of all, and gradually gets better when the model is adversarially trained.  }
\label{fig:acc-adv-all}
\end{figure}

\begin{figure}[t!]
        \centering
        \begin{subfigure}[b]{0.32\textwidth}
            \centering
            \includegraphics[width=\textwidth]{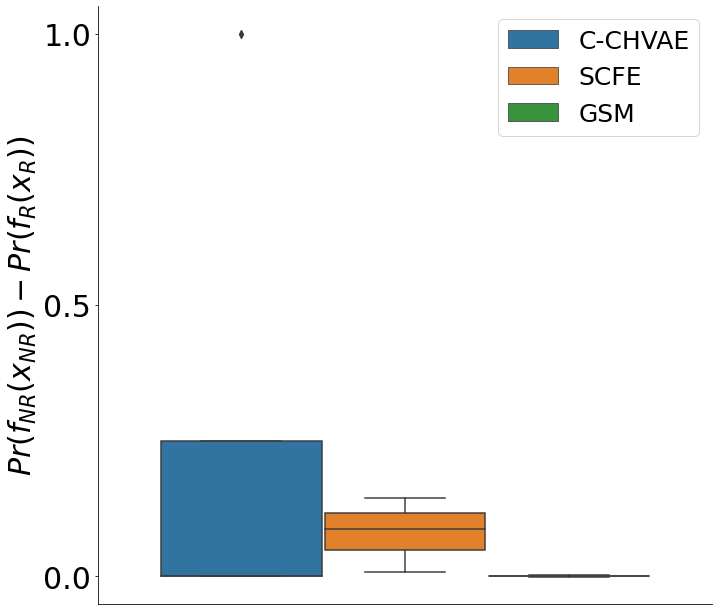}
            \caption{$\epsilon = 0.05$}
            \label{appn:val-eps-1}
        \end{subfigure}
        \begin{subfigure}[b]{0.32\textwidth}
            \centering
            \includegraphics[width=\textwidth]{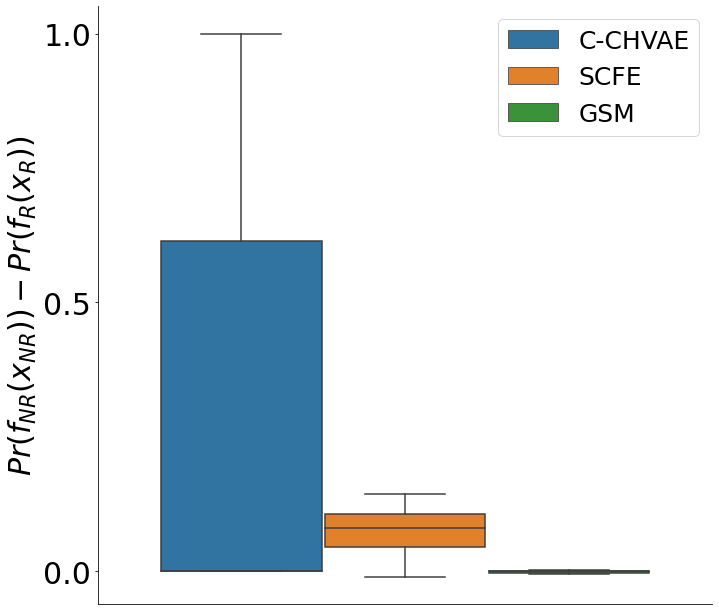}
            \caption{$\epsilon = 0.2$}
            \label{appn:tval-eps-2}
        \end{subfigure}
        \begin{subfigure}[b]{0.32\textwidth}
            \centering
            \includegraphics[width=\textwidth]{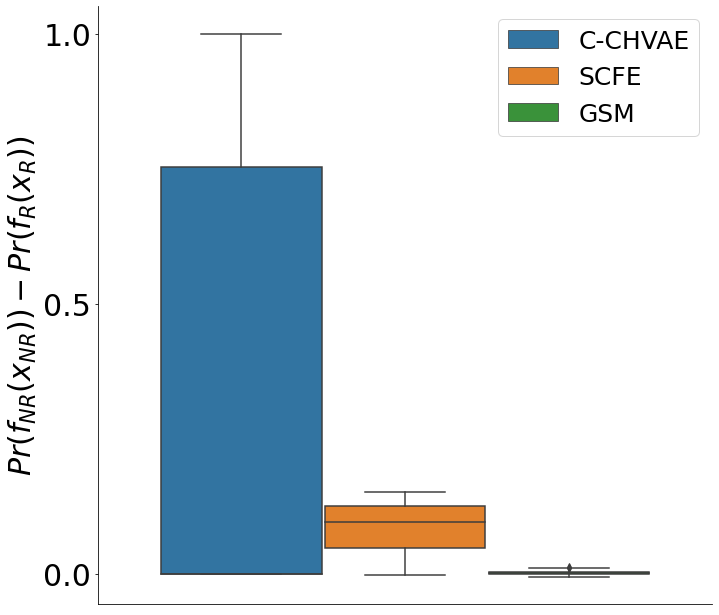}
            \caption{$\epsilon = 0.3$}
            \label{appn:val-eps-3}
        \end{subfigure}
        \caption{ Comparison between the validity of recourses generated for \vanilla and \robust model for varying degrees of robustness ($\epsilon$) for Adult dataset.  }
        \label{appn:theory-validation}
\end{figure}

\begin{figure}[t!]
        \centering
        \begin{subfigure}[b]{0.27\textwidth}
            \centering
            \includegraphics[width=\textwidth]{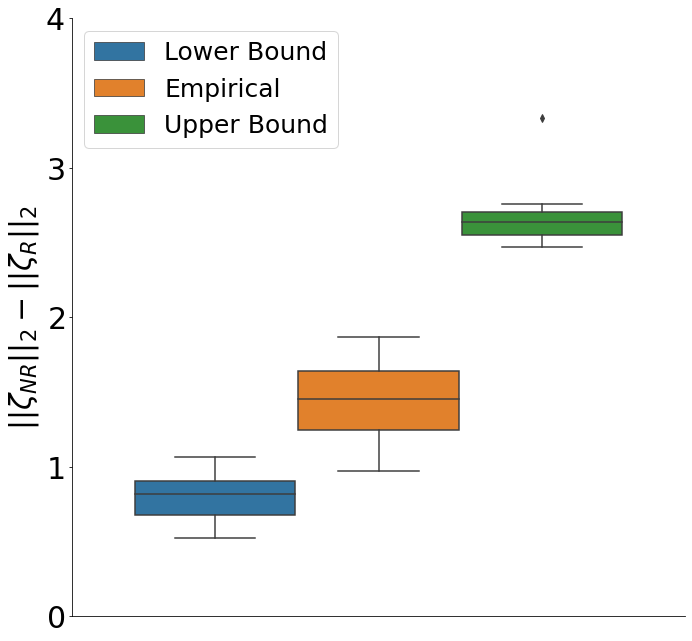}
            \caption{}
            \label{fig:theory-validity-non-linear}
        \end{subfigure}
        \begin{subfigure}[b]{0.27\textwidth}
            \centering
            \includegraphics[width=\textwidth]{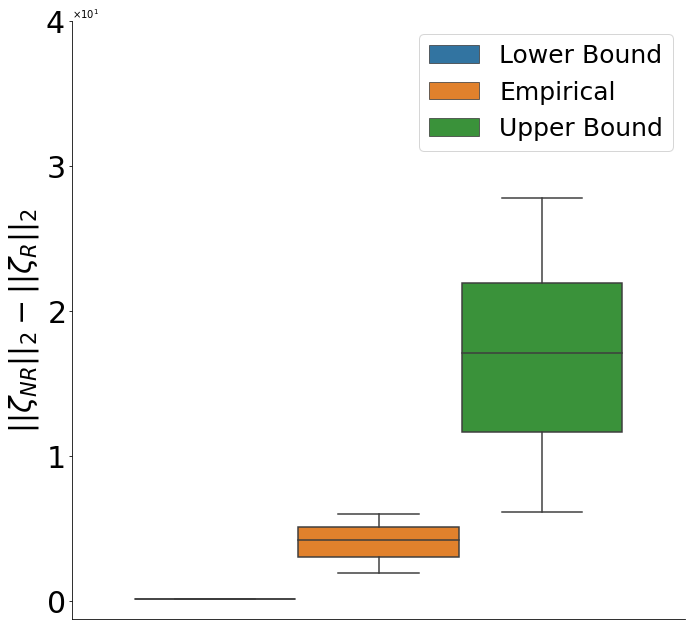}
            \caption{}
            \label{fig:theory-cost}
        \end{subfigure}
        \begin{subfigure}[b]{0.27\textwidth}
            \centering
            \includegraphics[width=\textwidth]{results/adult/adult_scfe_validity_theoretical.png}
            \caption{}
            \label{fig:theory-validity-non-linear}
        \end{subfigure}
        \caption{
        (a) Empirically calculated cost differences (in orange) for the original model and our theoretical lower (in blue) and upper (in green) bounds for SCFE recourses corresponding to \robust (trained using $\epsilon{=}0.3$) vs. \vanilla neural networks corresponding to test samples of the Adult dataset, based on Theorem \ref{thm:cost-bound-non-linear-sketch}. (b) Empirically calculated cost differences (in orange) for the original model and our theoretical lower (in blue) and upper (in green) bounds for SCFE recourses corresponding to \robust (trained using $\epsilon{=}0.3$) vs. \vanilla linear approximation of neural networks corresponding to test samples of the Adult dataset. Figure (c) is the empirical difference between the validity of recourses for \vanilla and \robust linear approximated model.  Results show no violations of our theoretical bounds.
        }
        \label{fig:theory-validation-bound-nn}
\end{figure}

\subsection{Analysis on Larger Neural Networks}
\label{app:wide_nn}

We show the impact on cost difference and validity with the increasing size of neural networks used to train \vanilla and \robust models in Figure \ref{fig:all-cost-non-linear-large}. 

\begin{figure*}[h]
        \begin{flushleft}
            \footnotesize
            \hspace{2.5cm}Cost Differences\hspace{5cm}Validity
        \end{flushleft}
        \begin{flushleft}
            \footnotesize
            \hspace{1.3cm}Depth\hspace{3.0cm}Width\hspace{2.5cm}Depth\hspace{2.7cm}Width
        \end{flushleft}
        \centering
        \begin{subfigure}[b]{0.24\textwidth}
            \centering
            \includegraphics[width=\textwidth]{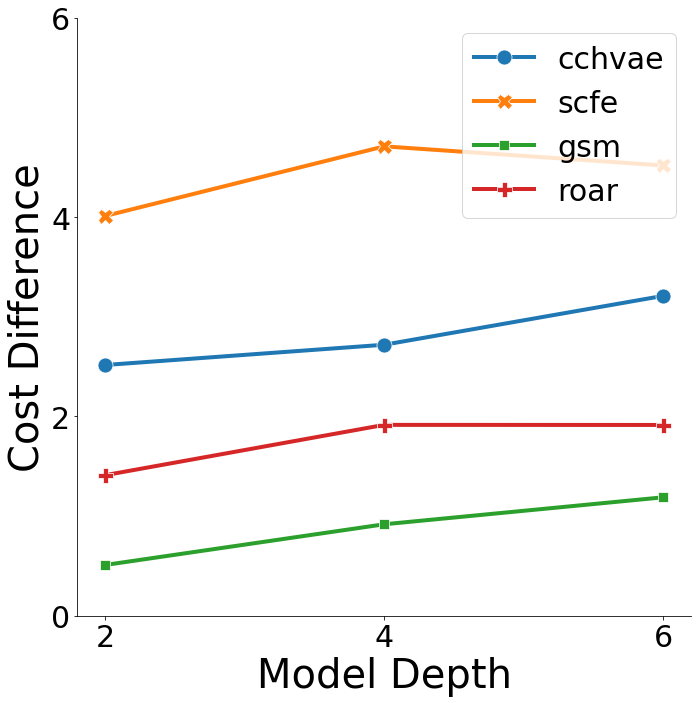}
            \label{fig:cost-adult}
        \end{subfigure}
        \begin{subfigure}[b]{0.24\textwidth}
            \centering
            \includegraphics[width=\textwidth]{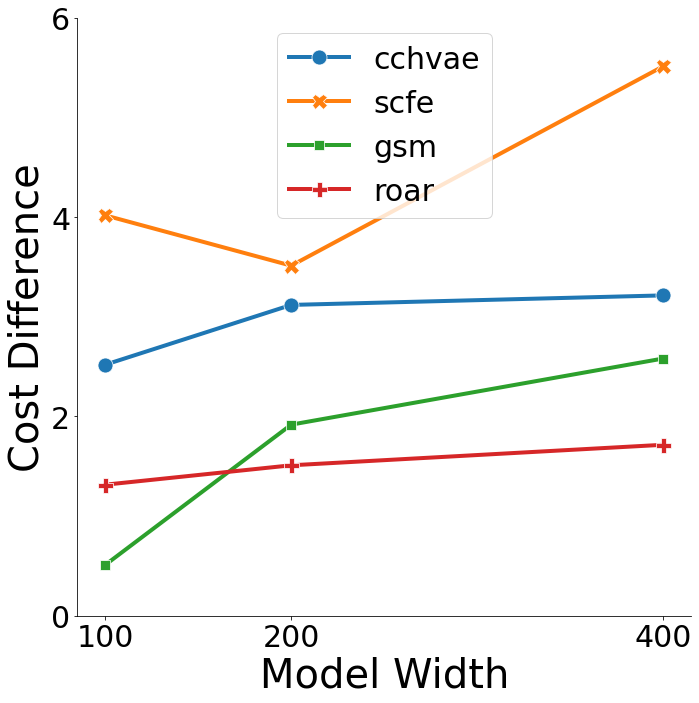}
            \label{fig:cost-compas}
        \end{subfigure}
        \begin{subfigure}[b]{0.24\textwidth}
            \centering
            \includegraphics[width=\textwidth]{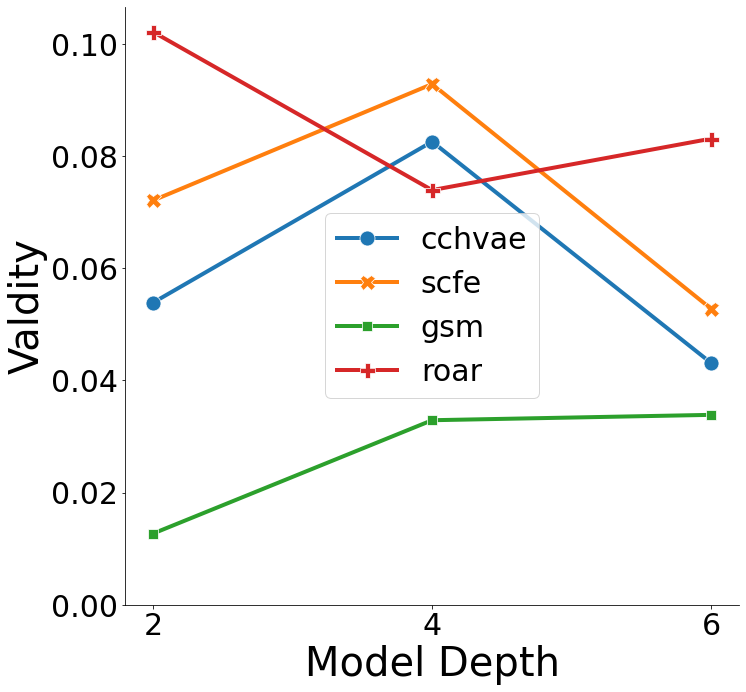}
            \label{fig:validity-adult}
        \end{subfigure}
        \begin{subfigure}[b]{0.24\textwidth}
            \centering
            \includegraphics[width=\textwidth]{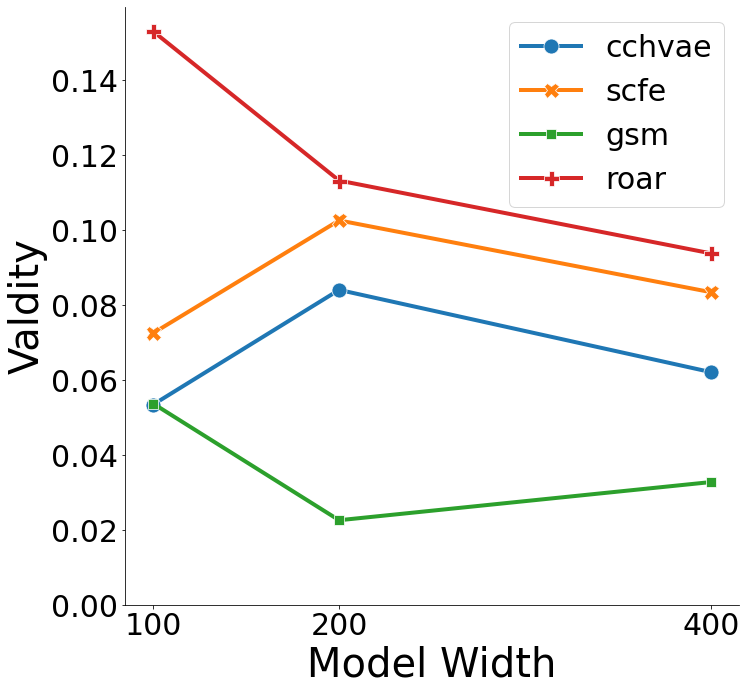}
            \label{fig:validity-compas}
        \end{subfigure}
        \caption{This figure analyzes the cost and validity differences between recourses generated using \vanilla and \robust neural networks trained on the Adult dataset. These differences are examined as the model size increases in terms of depth (defined as the number of hidden layers) and width (defined as the number of nodes in each hidden layer in a neural network of depth=2). Our findings suggest that: i) the cost difference (i.e., $\ell_{2}-$norm) between the recourses generated for \vanilla and \robust models remains consistent even as the model's depth or width increases, and ii) the validity of the recourses remains consistent even as the model's depth or width increases. Here, the \robust model is trained with $\epsilon = 0.3$.
        }
        \label{fig:all-cost-non-linear-large}
\end{figure*}

\newpage
\section{Broader Impact and Limitations}
\label{app:limitations}

Our theoretical and empirical analysis in understanding the impact of adversarially robust models on actionable explanations can be used by machine learning (ML) practitioners and model developers to analyze the trade-off between achieving adversarial robustness in ML models and providing reliable algorithmic recourses. The theoretical bounds derived in this work raise fundamental questions about designing and developing algorithmic recourse techniques for \vanilla and \robust models and identifying conditions when the generated recourses are invalid and, thus, unreliable. Furthermore, our empirical analysis can be used to navigate the trade-offs between adversarial robustness and actionable recourses when the underlying models are \vanilla vs. \robust before deploying them to real-world high-stake applications. 

Our work lies at the intersection of two broad research areas in trustworthy machine learning -- adversarial robustness and algorithmic recourse, which have their own set of pitfalls. For instance, state-of-the-art algorithmic recourse techniques are often unreliable and generate counterfactuals similar to adversarial examples that force an ML model to generate adversary-selected outputs. It is thus critical to be aware of these similarities before relying on recourse techniques to guide decision-makers in high-risk applications like healthcare, criminal justice, and credit lending.

\newpage
\begin{figure}[t!]
        \centering
        \begin{subfigure}[b]{0.45\textwidth}
            \centering
            \includegraphics[width=\textwidth]{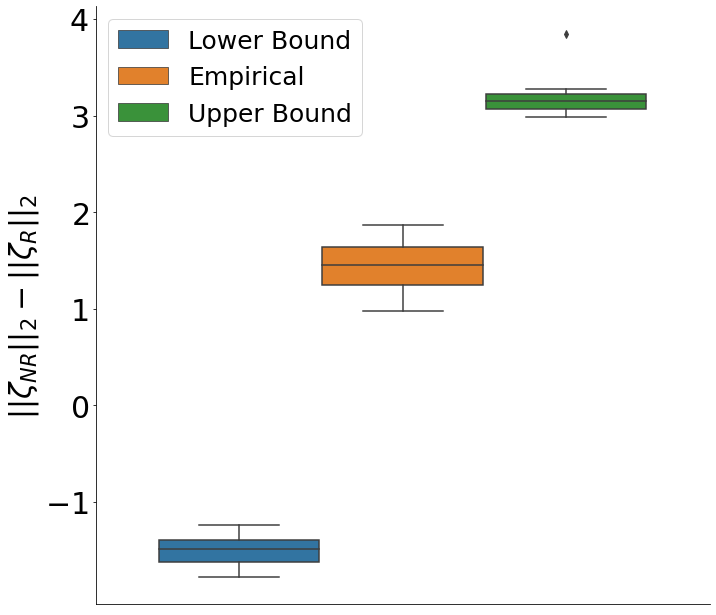}
            \caption{}
            \label{fig:theory-validity-non-linear}
        \end{subfigure}
        \begin{subfigure}[b]{0.45\textwidth}
            \centering
            \includegraphics[width=\textwidth]{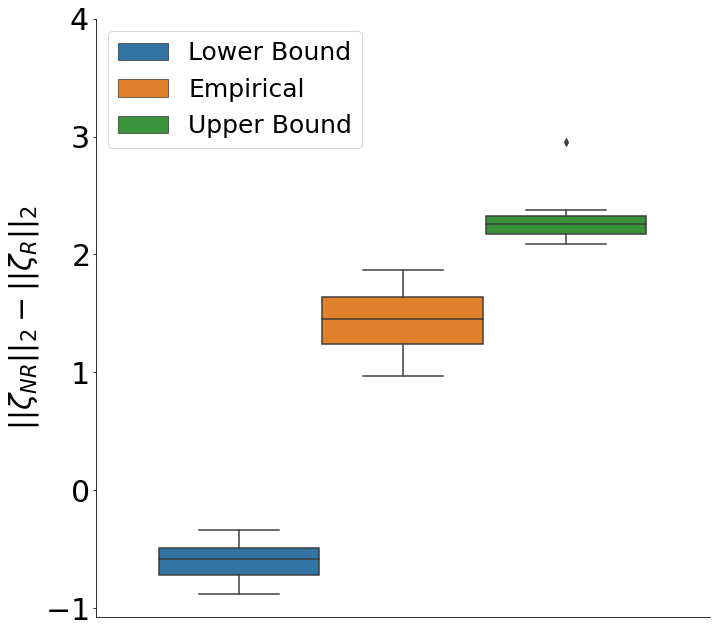}
            \caption{}
            \label{fig:theory-cost}
        \end{subfigure}
        \caption{
        (a) Empirically calculated cost differences (in orange) for the original model and our theoretical lower (in blue) and upper (in green) bounds for SCFE recourses corresponding to \robust (trained using $\epsilon{=}0.1$) vs. \vanilla neural networks corresponding to test samples of the Adult dataset, based on Theorem \ref{thm:cost-bound-non-linear-sketch}. (b) Empirically calculated cost differences (in orange) for the original model and our theoretical lower (in blue) and upper (in green) bounds for SCFE recourses corresponding to \robust (trained using $\epsilon{=}0.2$) vs. \vanilla neural networks corresponding to test samples of the Adult dataset, based on Theorem \ref{thm:cost-bound-non-linear-sketch}.
        }
        \label{fig:theory-validation-bound-nn}
\end{figure}

\begin{figure}
    \centering
    \includegraphics[width=0.63\textwidth]{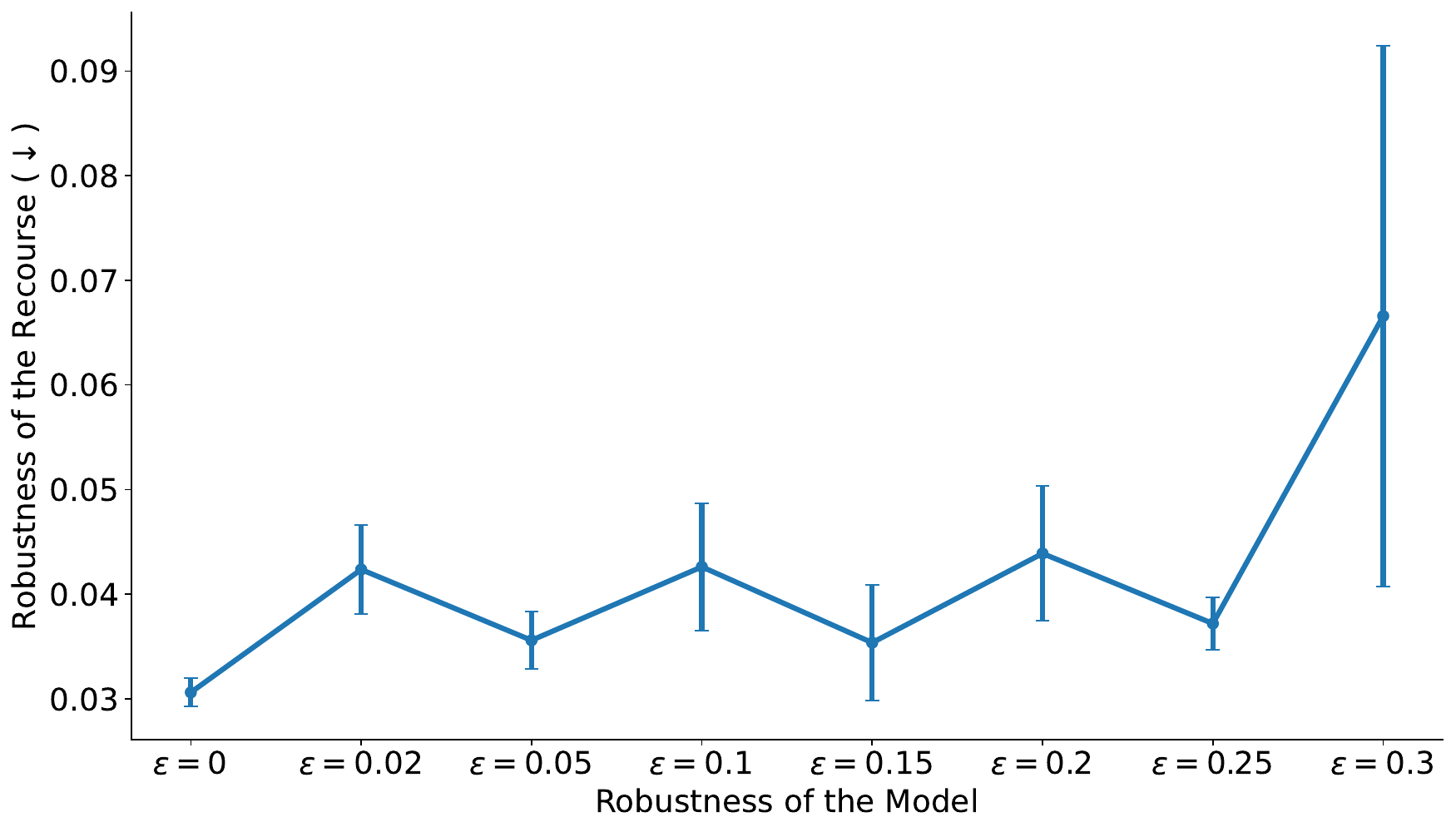}
    \caption{Analyzing the robustness of recourses generated using SCFE with increasing degree of robustness ($\epsilon$) for adversarially trained robust models for the COMPAS dataset. We find that the robustness of the recourse decreases as we increase the robustness of the underlying neural network model.}
    \label{fig:rebuttal}
\end{figure}

\section*{Acknowledgements}
This work is supported in part by the NSF awards IIS-2008461, IIS-2040989, IIS-2238714, and faculty research awards from Google, Adobe, JPMorgan, Harvard Data Science Initiative, and the Digital, Data, and Design (D\texttt{\^}3) Institute at Harvard. The views expressed here are those of the authors and do not reflect the official policy or position of the funding agencies.